\documentclass[11pt]{article}
\usepackage{amsmath,amssymb,amsthm}
\usepackage{graphicx}
\usepackage{geometry}
\usepackage{algorithm}
\usepackage[noend]{algpseudocode}
\usepackage{xspace}
\usepackage{listings}
\usepackage[
pagebackref,colorlinks=true,urlcolor=blue,
linkcolor=blue,citecolor=blue
]{hyperref}
\usepackage{aliascnt}
\usepackage[nameinlink]{cleveref}

\Crefname{ALC@unique}{Line}{Lines}
\crefname{ALC@unique}{line}{lines}
\makeatletter
\AtBeginEnvironment{algorithmic}{\let\c@ALC@unique\c@ALG@line}
\makeatother
\algrenewcommand{\textproc}{\textnormal}
\makeatletter
\providecommand*{\theHALG@line}{}
\renewcommand*{\theHALG@line}{\thealgorithm.\arabic{ALG@line}}
\makeatother

\newtheorem{theorem}{Theorem}[section]
\newaliascnt{lemma}{theorem}
\newtheorem{lemma}[lemma]{Lemma}
\aliascntresetthe{lemma}
\newaliascnt{proposition}{theorem}
\newtheorem{proposition}[proposition]{Proposition}
\aliascntresetthe{proposition}
\newaliascnt{corollary}{theorem}

\aliascntresetthe{corollary}
\theoremstyle{definition}
\newaliascnt{definition}{theorem}
\newtheorem{definition}[definition]{Definition}
\aliascntresetthe{definition}
\theoremstyle{remark}
\newaliascnt{remark}{theorem}
\newtheorem{remark}[remark]{Remark}
\aliascntresetthe{remark}

\crefname{theorem}{theorem}{theorems}
\Crefname{theorem}{Theorem}{Theorems}
\crefname{lemma}{lemma}{lemmas}
\Crefname{lemma}{Lemma}{Lemmas}
\crefname{proposition}{proposition}{propositions}
\Crefname{proposition}{Proposition}{Propositions}
\crefname{corollary}{corollary}{corollaries}
\Crefname{corollary}{Corollary}{Corollaries}
\crefname{definition}{definition}{definitions}
\Crefname{definition}{Definition}{Definitions}
\crefname{remark}{remark}{remarks}
\Crefname{remark}{Remark}{Remarks}

\DeclareMathOperator{\opt}{opt}
\DeclareMathOperator{\interval}{interval}
\DeclareMathOperator{\probub}{prob}
\DeclareMathOperator{\biasub}{bias}
\DeclareMathOperator{\cellub}{cell}
\DeclareMathOperator{\calerr}{calerr}

\newcommand{\PlayerP}{Player-P\xspace}
\newcommand{\PlayerL}{Player-L\xspace}

\newcommand{\AlgA}{\textsf{A}\xspace}
\newcommand{\AlgB}{\textsf{B}\xspace}
\newcommand{\initialize}{\textsf{initialize}\xspace}
\newcommand{\placeSign}{\textsf{label}\xspace}
\newcommand{\simulateGame}{\texttt{simulateGame}}

\newcommand{\currentB}{\text{currentB}\xspace}
\newcommand{\cnt}{\text{count}\xspace}

\newcommand{\cntHalf}{\text{countHalf}}
\newcommand{\phase}{\text{phase}}
\newcommand{\half}{\text{half}}
\newcommand{\shiftedA}{\textsf{shiftedA}}

\newcommand{\remainingSigns}{\textsf{remainingSigns}}
\newcommand{\persistentSigns}{\textsf{persistentSigns}}

\newcommand{\executionSteps}{\textsf{executionSteps}}

\title{Explicit Asymptotic Bounds for Sequential Calibration Beyond $T^{2/3}$}
\author{%
Eric Dai\\
{\small Edison Academy Magnet School}\\
{\small\href{mailto:eric.dai.kj@gmail.com}{\texttt{eric.dai.kj@gmail.com}}}
\and
Maxwell Fishelson\\
{\small Institute for Advanced Study}\\
{\small\href{mailto:maxfish@ias.edu}{\texttt{maxfish@ias.edu}}}
}
\date{October 7, 2026}

\begin{document}
  \maketitle

  \begin{abstract}
  Probability forecasts are \emph{calibrated} when predicted probabilities
  match empirical outcome frequencies: among events assigned a
  probability $p$, we'd hope that the fraction of positive outcomes is close to $p$.
  We study the problem of sequential forecasting of binary outcomes.
  The classical $O(T^{2/3})$ bound on expected
  cumulative $\ell_1$-calibration error established by Foster and Vohra
  \cite{foster1998} stood for over two decades
  until Dagan et al.\ \cite{dagan2024} reduced the exponent $2/3$ by an
  unspecified constant.
  We establish a new two-phase recursive labeling strategy for the
  sign-preservation-with-reuse game that yields the bound
  $O(n^{\alpha}t^\beta)$
  for all choices of space and time. We then sharpen the reduction from upper
  bounds on sign preservation to calibration by modifying the equivalence of
  Dagan et al.\ to use only $O(\log T)$ instances of the
  sign-preservation-with-reuse game.
  This lets us establish an explicit bound of $O(T^{0.662942288})$ as the first explicit exponent below $2/3$ for sequential calibration by combining both improvements and choosing explicit feasible parameters.
\end{abstract}

  \section{Introduction}
\label{sec:introduction}

Calibration is a fundamental criterion in statistical theory with
well-known theoretical and empirical foundations \cite{dawid1982bayesian}.
A forecaster is \emph{calibrated} if its probability forecasts agree with the
frequencies of observed outcomes. For example,
on the days when a calibrated weather forecaster reports a $30\%$ chance of
rain, it should rain on approximately $30\%$ of those days. Therefore,
calibration is a
natural desideratum for probability forecasts.

We study the classical \emph{sequential calibration} problem.
The problem describes a two-player game between a \emph{forecaster} and an
\emph{adversary} who interact for $T$ time steps.
At each time step $t \in [T]$, the
forecaster chooses a prediction $p_t \in [0,1]$ and the adversary
simultaneously chooses an outcome $y_t \in \{0,1\}$,
and then both choices are revealed. Each player may randomize and use the
history of past
predictions and outcomes, but the adversary cannot observe the current
realized prediction before choosing $y_t$. The prediction $p_t$
represents the forecaster's probability for the outcome $y_t = 1$.

Let $n_T(p)$ be the number of time steps
on which the prediction is $p$, and let $m_T(p)$ be the number of those
time steps with outcome $1$. The \emph{cumulative $\ell_1$-calibration error} is
\begin{equation*}
  \calerr(T) := \sum_{p \in [0,1]}|m_T(p) - p \cdot n_T(p)|.
\end{equation*}

Foster and Vohra \cite{foster1998} were the first to provide the classical upper
bound $O(T^{2/3})$
on expected cumulative calibration error. Following their work,
the question of whether the
exponent $2/3$ could be improved remained unresolved for more than two
decades. Recently, Dagan et al.\ \cite{dagan2024} broke the $2/3$ barrier by
proving a bound of
$O(T^{2/3 - \eta})$ for some absolute constant $\eta > 0$. However,
their proof does not
state a specific number for $\eta$, and thus does not explicitly state an
exponent below $2/3$ for calibration error.
This motivates the quantitative question: how much can the exponent $2/3$ be
reduced in the sequential calibration
problem?

\subsection{Our contributions}

We obtain the following explicit upper bound on expected cumulative
calibration error.
The exponent is an improvement of more than $0.00372$ over the classical upper
bound of $2/3$.
\begin{theorem}[Main result, informal]
  \label{thm:intro-calibration}
  For every horizon $T\geq1$, there exists a randomized forecaster whose
  expected cumulative calibration error against any adversary is
  \[
    \mathbb E[\calerr(T)]
    = O\!\left(T^{0.662942288}\right).
  \]
\end{theorem}

\paragraph{Related work and approach.}

The key technical idea is to reduce calibration to a combinatorial
game. Qiao and Valiant \cite{qiao2021} introduced the \emph{sign-preservation
game} to prove
lower bounds on calibration error. Dagan et al.\ \cite{dagan2024} introduced
the \emph{sign-preservation-with-reuse (SPR) game} and established its
connection
to calibration upper bounds. SPR takes place on an ordered board of
cells. In each round, one player, the \emph{pointer}, selects an empty cell, and a
second player, the
\emph{labeler}, places a plus or minus sign there. The labeler may
remove minus signs to the left of the selected cell and plus signs to
its right.
The labeler and the pointer seek to minimize and maximize the number of signs
remaining on the board, respectively.
Dagan et al.\ show that given a labeling strategy that preserves
$O(n^{1 - \varepsilon})$ signs after at most $n$ rounds on $n$ cells,
there exists a forecaster with expected calibration error
$O(T^{2/3 - \varepsilon/18})$.

This reduction from calibration to a labeling strategy in SPR uses the minimax
approach of Hart \cite{hart2023minimax}. His proof establishes the classical
$O(T^{2/3})$ bound of Foster and Vohra \cite{foster1998} by first allowing the
forecaster
to know the adversary's conditional probability of the next outcome,
and then applying minimax to obtain a strategy without this information.
Dagan et al.\ use SPR to improve the bound in this argument and break the $2/3$
barrier. One limitation of this approach is that the minimax step does not give
an efficient forecasting algorithm.
Zhang \cite{zhang2026efficient} gave a polynomial-time
randomized forecaster with expected calibration error
$O(T^{2/3 - \eta})$ for some constant $\eta > 0$, combining the
SPR forecasting procedure with a correction based on Blackwell's approachability
theorem
that controls the error from using surrogate conditional means.
Our contribution focuses on improving the explicit upper bound on the
calibration exponent.

\paragraph{An improved labeling strategy.}
We give a two-phase recursive labeling strategy,
building on the recursive bias ideas of Dagan et al.\ \cite{dagan2024}.
The strategy first runs
on the two halves of an interval until their request counts satisfy a
balance condition. When the condition is satisfied,
the strategy continues with an adjusted
bias toward one sign type. The analysis allows us to establish separate bounds
for newly placed signs and signs that survive from
earlier epochs.
This labeling strategy yields bounds that
depend separately on the number of cells and the number of rounds.

\paragraph{A sharper passage to calibration.}
We modify the forecasting construction in the reduction of
Dagan et al.\ \cite{dagan2024}. We analyze this construction using the labeling
bound $O(n^\alpha t^\beta)$ at each partition scale.
Our construction uses only $O(\log T)$ SPR
instances, compared with the $O((\log T)^2)$ instances used by Dagan et al.\
\cite{dagan2024}.
For $\alpha \in (0,1)$, $\beta \in (0,1)$ and
$\varepsilon = 1 - \alpha - \beta > 0$, our construction obtains an expected
calibration error of
\[
  O\!\left(T^{\frac{2 + \varepsilon}{3 + 2\varepsilon}}\right)
  = O\!\left(T^{\frac23 - \frac{\varepsilon}{9 + 6\varepsilon}}\right).
\]
The analysis of this construction sums the conditional-means contributions of
the SPR instances across different space and bias scales,
and then balances their total against the fluctuations of outcomes around
their conditional means.

\paragraph{Explicit parameters.}
Our final calibration exponent depends on the parameters chosen for the
labeling strategy. We prove that the parameters we choose satisfy certain
inequalities that are sufficient to establish our upper bound for SPR.
The resulting game bound is
\[
  O\!\left(n^{0.0082540892}t^{0.9574603467}\right).
\]
Applying the reduction with these exponents proves
\Cref{thm:intro-calibration}.

  \section{Preliminaries}
\label{sec:preliminaries}

In this section, we formally introduce and define sequential calibration and
the sign-preservation-with-reuse game. For an integer $n\geq1$,
write $[n] := \{1,\dots,n\}$.
Unless otherwise specified, horizons, board sizes, and round counts are
integers.

\subsection{Sequential calibration}
\label{sec:sequential-calibration}

\emph{Sequential calibration} is a theoretical framework where the task of making reliable predictions over time is modeled as a repeated online game between a forecaster and an
adversary.
Fix a horizon $T\geq1$ and a finite prediction set
$P\subset[0,1]$ which the forecaster chooses. At each time step $t \in [T]$, the forecaster predicts
$p_t \in P$ and the adversary simultaneously chooses an outcome $y_t \in \{0,1\}$,
meaning the adversary cannot observe the
current prediction $p_t$ before choosing $y_t$.
Both players may randomize and use the history of the game
$H_t = ((p_i,y_i))_{i = 1}^{t - 1}$, but the forecaster does not know the
adversary's conditional probability of choosing $y_t = 1$ given that history.

For $p \in P$, let $n_t(p) = \sum_{i = 1}^t\mathbf{1}\{p_i = p\}$ be the
number of predictions at $p$, and let
$m_t(p) = \sum_{i = 1}^t y_i\mathbf{1}\{p_i = p\}$ count their positive
outcomes. The \emph{cumulative $\ell_1$-calibration error} is
\[
  \calerr(t) = \sum_{p \in P}|m_t(p) - n_t(p) \cdot p|.
\]
We wish to obtain a bound on $\mathbb E[\calerr(T)]$ against every
adversary, where the expectation is over both players' randomness.

\subsection{Sign-preservation-with-reuse game}
\label{sec:spr}

The sign-preservation game, introduced by Qiao and Valiant \cite{qiao2021}, is a two-player combinatorial game used to model a simplified version of sequential calibration. In this paper, we consider the generalized version of the sign-preservation game, which Dagan et al.\
\cite{dagan2024} introduce as the \emph{sign-preservation-with-reuse game}. Fix integers $n\geq1$ and $t\geq0$.
There are initially $n$ empty cells, numbered $1$, \dots, $n$, and two
players: \PlayerP\ (the \emph{pointer}) and \PlayerL\ (the
\emph{labeler}). Each of at most $t$ rounds proceeds as follows:
\begin{enumerate}
  \item The pointer may terminate the game. Otherwise, it selects
  an empty cell $j \in [n]$.
  \item The labeler may remove any subset of the minus signs to
  the left of $j$ and the plus signs to the right of $j$.
  \item The labeler places either a plus or a minus in cell $j$.
\end{enumerate}
Sign-preservation-with-reuse differs from sign-preservation by allowing the pointer to select cells that have been emptied before. The game ends when the pointer stops, all $t$ rounds have occurred,
or no empty cell remains. The pointer and labeler seek to maximize and minimize
the number of signs
remaining, respectively. Both players may observe the
full history of the game and randomize in their strategies.

Let $\opt(n,t)$ denote the number of signs remaining at the end of a game of SPR
with $n$ cells and $t$ rounds under optimal play.
Qiao and Valiant \cite{qiao2021} described strategies for the pointer in order to prove
calibration lower bounds. Conversely, we and Dagan et al.\
\cite{dagan2024} describe strategies for the labeler, yielding upper
bounds on $\opt(n,t)$ and calibration error.

  \section{Overview of \texorpdfstring{\Cref{thm:intro-calibration}}{Theorem \getrefnumber{thm:intro-calibration}}}
\label{sec:proof-overview}

In this section, we outline the proof of \Cref{thm:intro-calibration}, which is broken down into three main parts.
\Cref{sec:signflipping-ub} describes the new labeler strategy for the
sign-preservation-with-reuse game. \Cref{sec:optimizing-calibration} describes
the improved reduction from calibration to the sign-preservation-with-reuse
game. \Cref{sec:explicit-parameters} provides the specific parameters used to obtain the final calibration exponent.

\subsection{Overview of \texorpdfstring{\Cref{sec:signflipping-ub}}{Section \getrefnumber{sec:signflipping-ub}}}

We seek the upper bound on $\opt(n, t)$ of \Cref{thm:real-root-bound},
\begin{equation}
  \opt(n, t) = O(n^\alpha t^\beta)
  \label{eq:overview-target}
\end{equation}
with $\varepsilon := 1 - \alpha - \beta > 0$, by choosing the parameters from an admissible tuple as in \Cref{def:real-admissible}.
The definition is involved because it collects the parameter constraints
needed in the different cases of the inductive proof.
For now, assume $n$ is a power of two. Note that simply maintaining independent labelers on
the two halves does not suffice: if each labeler receives $t/2$ requests,
summing their bounds gives
\begin{equation}
  2(n/2)^\alpha(t/2)^\beta
  = 2^{1 - \alpha - \beta}n^\alpha t^\beta.
  \label{eq:overview-naive-split}
\end{equation}
Since $2^{1 - \alpha - \beta} > 1$, \eqref{eq:overview-naive-split} fails to close the induction.
Instead, we must account for deletions between the two halves to obtain stronger
bounds.

Our labeling strategy deletes every minus to the left of a request and every
plus to its right, which affects the balance of positive and negative signs
within the game. Following Dagan et al.\ \cite{dagan2024}, we use
recursive bias to exploit this asymmetry: we wish to show an instance with integer bias parameter $b$
leaves at most
\begin{equation}
  C\lambda^{b\sigma}n^\alpha t^\beta
  \label{eq:overview-bias-bound}
\end{equation}
surviving signs of each type $\sigma \in \{-1,+1\}$ among those it placed,
where $C$ is a constant and $\lambda > 1$.
We can see that by changing $b$ by one, we trade off the bounds in \eqref{eq:overview-bias-bound} on the two sign types. For example,
decreasing $b$ by one decreases
the plus bound by a factor of $\lambda$ while increasing the minus
bound by a factor of $\lambda$. We will show that \eqref{eq:overview-bias-bound} holds by induction on space and strongly on time at each fixed space.

Each epoch first runs instances in phase~1 on the two halves.
Write $t_h \geq t_l$ for their request counts, denoting the number of times the \emph{heavy half} and \emph{light half} were requested, respectively, and $\tau$ for the requests
in earlier epochs. By the inductive hypothesis, the two children's contributions are bounded by
\begin{equation}
  C\lambda^{b\sigma}(n/2)^\alpha t_h^\beta
  \quad\text{and}\quad
  C\lambda^{b\sigma}(n/2)^\alpha t_l^\beta.
  \label{eq:overview-child-bounds}
\end{equation}

Since some signs survive between epochs, treating each epoch separately gives
too weak a bound. Instead,
we maintain a bound for all remaining signs
and a stronger bound for \emph{persistent signs}.
Persistent signs are pluses in the left
half of an interval and minuses in the right half.
Opposite-half requests cannot delete persistent signs, and once both halves
have been visited in the current epoch, only persistent signs from earlier
epochs can remain. At each restart, for each sign type $\sigma$, we bound persistent and total signs,
respectively, by
\begin{equation}
  C\lambda^{b\sigma}(n/2)^\alpha K_p\tau^\beta
  \quad\text{and}\quad
  C\lambda^{b\sigma}(n/2)^\alpha K_t\tau^\beta,
  \label{eq:overview-inherited-bounds}
\end{equation}
where $0<K_p<1$ and $K_t>K_p$.

Admissibility ensures that the contributions from the current epoch and
earlier epochs fit within the inductive bounds. Although this requires
several constraints, our simpler two-phase strategy makes them tractable
enough to search for feasible parameters numerically and verify them
rigorously. \Cref{sec:explicit-parameters} gives an explicit tuple, whose
admissibility we establish in \Cref{app:parameter-verification}.
We distinguish three cases according to the state of the current epoch.

\paragraph{Case 1: phase~1, only one half has been visited.}
Here $t_l=0$ and $t=\tau+t_h$. Both persistent and non-persistent signs
from earlier epochs may remain, so we use the total-sign bound in
\eqref{eq:overview-inherited-bounds}. Adding the visited child's
contribution from \eqref{eq:overview-child-bounds} gives
\begin{equation}
  C\lambda^{b\sigma}(n/2)^\alpha(K_t\tau^\beta+t_h^\beta)
  \leq C\lambda^{b\sigma}(n/2)^\alpha Dt^\beta.
  \label{eq:overview-one-half}
\end{equation}
Thus, the admissibility condition $D<2^\alpha$ makes \eqref{eq:overview-one-half}
smaller than the desired bound \eqref{eq:overview-bias-bound}.

\paragraph{Case 2: phase~1, both halves have been visited.}
Only persistent signs from earlier epochs remain, so the coefficient
$K_t$ in \eqref{eq:overview-one-half} is replaced by $K_p$,
while both contributions in \eqref{eq:overview-child-bounds} must be counted.
Since $0<\beta<1$, for a fixed sum $t_h+t_l$, the expression $t_h^\beta+t_l^\beta$ is
largest when the counts are equal. Thus, while $t_l < \rho t_h$,
the imbalance improves the bound, and we can continue using the two
half-interval instances.
If $\delta t_l < \tau$, the smaller coefficient $K_p$ in
\eqref{eq:overview-inherited-bounds} supplies the saving
even when the two half-counts are balanced.
These two sources of saving explain why phase~1 continues until
both $t_l\geq\rho t_h$ and $\delta t_l\geq\tau$.
While either phase~1 condition holds, admissibility gives, with $t=\tau+t_h+t_l$,
\begin{equation}
  C\lambda^{b\sigma}(n/2)^\alpha
  (K_p\tau^\beta+t_h^\beta+t_l^\beta)
  \leq C\lambda^{b\sigma}(n/2)^\alpha Mt^\beta.
  \label{eq:overview-both-halves}
\end{equation}
Thus, $M\leq D<2^\alpha$ makes \Cref{eq:overview-both-halves} smaller than the desired bound
\eqref{eq:overview-bias-bound}.

\paragraph{Case 3: phase~2.}
When the counts become sufficiently balanced and both are large relative
to the count from earlier epochs, we trigger phase~2.
Suppose the trigger lies in the right half, whose count
is $t_l$. It deletes every minus in the left half, removing the larger
child's contribution to the minus bound. The strategy then uses a
full-interval continuation with bias $b - 1$ for subsequent requests. After $t_c$ continuation requests,
its contributions to the plus and minus bounds are, respectively,
\begin{equation}
  C\lambda^{b-1}n^\alpha t_c^\beta
  \quad\text{and}\quad
  C\lambda^{-(b-1)}n^\alpha t_c^\beta.
  \label{eq:overview-continuation-bounds}
\end{equation}
The deleted signs provide room for the larger minus allowance in
\eqref{eq:overview-continuation-bounds}, while the plus bound benefits
from the smaller allowance.
This continuation lasts for $\lfloor\kappa t_h\rfloor$ requests in order to
limit the time for which the instance has changed bias.
A left-half trigger exchanges the roles of the signs.

At an epoch's end, the deletion and bias change ensure that the
persistent and total survivors satisfy \eqref{eq:overview-inherited-bounds} with the updated
request count. This allows the argument to continue across restarts.
The proof must also bound executions that stop before an epoch ends.
Combining \eqref{eq:overview-one-half}, \eqref{eq:overview-both-halves}, and
the phase~2 contributions in \eqref{eq:overview-continuation-bounds}
with the inherited bounds \eqref{eq:overview-inherited-bounds} gives,
for sufficiently large $t$, the bound
\begin{equation}
  C\lambda^{b\sigma}(n/2)^\alpha(Dt^\beta+1).
  \label{eq:overview-combined-bound}
\end{equation}
The admissibility condition $D<2^\alpha$ absorbs the additive term in
\eqref{eq:overview-combined-bound} for large $t$, while executions with small $t$ are handled by the constant $C$, proving the desired bound of \eqref{eq:overview-bias-bound}.
Summing over the two sign types
at bias zero and padding the board to a power of two then gives
the bound \eqref{eq:overview-target} of \Cref{thm:real-root-bound}
for arbitrary board sizes.

This simpler two-phase strategy replaces
the four-phase strategy of Dagan et al.\ \cite{dagan2024}.
\Cref{sec:explicit-parameters} gives an admissible parameter tuple with
$\varepsilon>0$, completing the claimed bound of \eqref{eq:overview-target}.
\subsection{Overview of \texorpdfstring{\Cref{sec:optimizing-calibration}}{Section \getrefnumber{sec:optimizing-calibration}}}

\Cref{thm:intro-reduction} converts the bound on the number of remaining signs in SPR into calibration
error
\begin{equation}
  O(T^{\frac{2 + \varepsilon}{3 + 2\varepsilon}}),
  \label{eq:overview-calibration-target}
\end{equation}
where
$\varepsilon = 1 - \alpha - \beta > 0$ is derived from an admissible tuple as
defined in \Cref{def:real-admissible}. We modify the forecasting construction
of Dagan et al.\ \cite{dagan2024} and analyze it using the separate space and
time
bounds for SPR, rather than only using the diagonal regime $t = n$.
As in their minimax approach, we first allow the forecaster
to observe the conditional mean $e_t$ of the next outcome.

For $T = 2^\tau$, the forecaster uses $h$ partition
levels starting at $i_0$. The coarsest level, $i = i_0$, uses several bias
thresholds, while each finer level, $i \in \{i_0 +1, \dots, i_0 + h-1\}$, uses just one bias threshold. At each level and bias threshold pair, we establish an even and odd instance of SPR, so that cells
in one game are separated. This construction uses only $O(\log T)$ SPR
instances, compared to the $O((\log T)^2)$ instances used by Dagan et al.\
\cite{dagan2024}.
At level $i$, each cell represents an interval of length $2^{-(i+1)}$,
and a threshold index $j$ specifies the \emph{bias threshold} $2^{j-i}$.
A plus selects the left endpoint, and a minus selects the right endpoint.
Each cell in each game keeps a separate bias record at each endpoint, adding $e_t-p_t$ whenever
that endpoint is selected for the cell. The record corresponding to the cell's current
sign is its \emph{active bias}, taken to be zero for an empty cell.

\paragraph{Bias removal.}
The forecaster first tries to reduce an active bias, in a process called \emph{bias removal}.
If the active bias is positive and $e_t$ lies to the left of its endpoint,
predicting that endpoint gives $e_t-p_t<0$ and reduces the bias.
Similarly, if the active bias is negative and $e_t$ lies to the right of its endpoint, predicting that endpoint gives $e_t - p_t > 0$ and reduces its magnitude. Bias removal is performed only when the
active bias magnitude exceeds one, so an update of magnitude at most one cannot reverse its sign.

The importance of bias removal is twofold. If bias removal is performed, then the total sum of magnitudes of bias across all cells in all games cannot increase. If bias removal is not performed, every sign that the placement request can legally delete has active bias of magnitude at most one.
\paragraph{Bias placement.}
If no removal is available, the forecaster
uses the first eligible game, in increasing order of $i$ and then $j$, to place bias by predicting the endpoint
specified by the cell's sign. A cell is \emph{eligible} when its interval contains $e_t$
and its active bias magnitude is below its bias threshold. 

If the cell is empty, we will simulate a round of the SPR instance, where the pointer player points towards this cell in the corresponding game, and the labeler will return its sign. Let $\sigma$ be the sign in this cell. Then, when we predict the endpoint corresponding to the chosen sign $\sigma$, the bias increment satisfies
\begin{equation}
  0\leq \sigma(e_t-p_t)\leq 2^{-(i+1)}.
  \label{eq:overview-placement-increment}
\end{equation}
Thus placement adds a small nonnegative amount of bias in the direction chosen by the labeler. If no game is eligible,
the forecaster rounds $e_t$ to the nearest prediction point on a fixed grid.

Since $\opt(n, t)$ depends on time $t$, it is important to bound the number of recorded rounds in a game.
To do this, we consider a \emph{reduced transcript}, undoing moves whose signs
are deleted by the next remaining move.
Except at the smallest coarsest-level threshold, between successive queries in this transcript to the same cell $c$ in the
same instance $G$, bias in another instance must be
rebuilt to its threshold, which takes time.
The increment bound \eqref{eq:overview-placement-increment} therefore
limits these transcripts' lengths, allowing the SPR guarantee to bound
the surviving signs.

\paragraph{From bias to calibration.}
To relate calibration error to bias, we use the decomposition
\begin{equation}
  y_t-p_t=(e_t-p_t)+(y_t-e_t).
  \label{eq:overview-calibration-split}
\end{equation}
After summing over rounds at each prediction and applying the triangle
inequality, \eqref{eq:overview-calibration-split} bounds calibration error
by a \emph{bias term} from $e_t-p_t$ and a \emph{variance term} from $y_t-e_t$.

To bound the bias term, note that the bias thresholds ensure an active bias has magnitude at most $2^{j-i}+1$, while an
inactive bias has magnitude at most one.
Thus, if $S_G$ signs remain in $G$ and $N_G$ cells have received placements,
the total absolute stored bias in that game is at most
\begin{equation}
  2^{j-i}S_G+2N_G.
  \label{eq:overview-stored-bias}
\end{equation}
The first term in \eqref{eq:overview-stored-bias} is bounded by multiplying
the number of remaining signs in each SPR instance by its bias threshold.
For the second term, calling a cell in a game with fine intervals
requires its parent to reach its bias threshold first.
By \eqref{eq:overview-placement-increment}, this requires many placements,
limiting both the number of used cells and therefore the number of distinct
predictions.
The contribution to bias from the fixed grid \emph{rounding fallback} is
$O(2^{\tau-i_0-h})$.

For the variance term, the increments $y_t-e_t$ have conditional
mean zero. The bound on used cells gives at most
$L=O(2^{i_0}+2^{\tau-i_0-h})$ distinct predictions, including those made
by rounding. A martingale second-moment bound and Cauchy--Schwarz bound the expected variance
term by $\sqrt{LT}$.

Combining these bounds and summing the threshold-weighted SPR guarantees
gives the calibration estimate.
Setting $h = \tau - 2i_0$ therefore gives
\begin{equation}
  \mathbb E[\calerr(T)]
  = O\!\left(2^{\tau - (1 + \varepsilon)i_0}
  +h2^{i_0} + 2^{(\tau + i_0)/2}\right).
  \label{eq:overview-calibration-balance}
\end{equation}
By setting $i_0 = \lfloor\tau/(3 + 2\varepsilon)\rfloor$, we balance the
first and third terms of \eqref{eq:overview-calibration-balance} up to constant factors,
which is sufficient because the second is smaller. The minimax step
removes the forecaster's access to conditional means, giving
\eqref{eq:overview-calibration-target} and completing the proof of
\Cref{thm:intro-reduction}. Its exponent depends only on
$\varepsilon$, so we choose admissible parameters to maximize
$1 - \alpha - \beta$. The explicit choice in \Cref{sec:explicit-parameters}
then gives the claimed calibration exponent.

  \section{An improved labeling strategy for the sign-preservation-with-reuse
game}
\label{sec:signflipping-ub}

In this section, we prove the following theorem, giving a strategy for
\PlayerL whose guarantee depends separately on the numbers of cells and rounds.

\begin{theorem}
  \label{thm:real-root-bound}
  For every admissible parameter tuple
  $(\alpha,\beta,\lambda,\rho,\kappa,\delta)$ in the sense of
  \Cref{def:real-admissible}, there is a constant $K > 0$ such that,
  for all integers $n \geq 1$ and $t \geq 0$,
  \[
    \opt(n,t) \leq K n^\alpha t^\beta.
  \]
\end{theorem}

We will prove \Cref{thm:real-root-bound} by establishing a per-sign
bound for the recursive strategy in \Cref{lem:real-labeling}, then
applying it to the root instance and padding to arbitrary board sizes.

\subsection{Description of the strategy}

\paragraph{Notation.}
The strategy for \PlayerL is given by two mutually recursive algorithms,
$\AlgA$ and $\AlgB$. The calls $\AlgA.\initialize(l,r,b)$ and
$\AlgB.\initialize(l,r,b,\tau)$ create instances with the specified
parameters and initialize their local state. Subsequent calls to
$\AlgA.\placeSign(s)$ or $\AlgB.\placeSign(s)$ process a request at cell
$s$ and return its label, except that $\AlgB.\placeSign(s)$ returns
$\perp$ when its epoch has ended.

An instance initialized on $[l,r]$ \emph{covers} $n := r - l + 1$ cells and has
\emph{bias} $b \in \mathbb Z$, where $l \in \mathbb N$,
$r \in \mathbb N$, and $l \leq r$.
For an instance $\mathcal{I}$ of either $\AlgA$ or $\AlgB$,
let $\executionSteps(\mathcal{I})$ denote the number of calls to
$\mathcal{I}.\placeSign$ that have returned a sign in $\{-1,1\}$;
calls returning $\perp$ are not counted. An \emph{epoch} of $\AlgA$
is the sequence of requests answered by one of its instances of $\AlgB$.
The epoch's parameter $\tau$ counts the number of requests answered in all earlier
epochs, and signs placed in earlier epochs that still remain are called
\emph{inherited signs}. Note that $\executionSteps(\AlgA)$ includes requests
from earlier epochs,
while $\executionSteps(\AlgB)$ counts only those in that epoch.

For $\sigma \in \{-1,1\}$,
let $\remainingSigns(\mathcal{I},\sigma)$ denote the number of remaining signs
of type $\sigma$ placed in its covered cells in response to requests made to $\mathcal{I}$, including recursive calls. For an instance covering at least two cells, a sign
is \emph{persistent} if it is a plus in the left half or a minus in
the right half. Notice that persistent signs
survive requests in the opposite half but can still be deleted by a
request in their own half. We write $\persistentSigns(\mathcal{I},\sigma)$ for
the number of such signs among those counted by
$\remainingSigns(\mathcal{I},\sigma)$.

We use the
``.'' notation for the local variables and routines of an instance,
as in $\AlgB.\phase$ and $\AlgA.\placeSign(s)$.
We use the convention $\operatorname{sign}(b) = +1$ for $b\geq0$ and
$\operatorname{sign}(b) = -1$ for $b < 0$,
so a zero bias also returns a legal sign, and positive bias favors plus signs.

\paragraph{Overview of the labeling routines.}
For now, to simplify the recursion, let $n$ be a power of two.
We initialize a root instance $\AlgA_0$ with
$l = 1$, $r = n$, $b = 0$. At each time step, after \PlayerP selects,
or \emph{requests}, an empty
cell $s \in [n]$, we delete every minus to the left of $s$ and every
plus to its right, and then place the sign returned by
$\AlgA_0.\placeSign(s)$. Each instance of $\AlgA$ covering more than
one cell passes the request to its current instance $\currentB$ of
$\AlgB$.

In phase~1, this instance passes the request to one of its
two children $\AlgA[0]$, $\AlgA[1]$, according to whether $s$ belongs to
the left or right half of its interval. In phase~2, it instead passes
the request to a new instance $\shiftedA$ of $\AlgA$, which covers the
same interval and has bias one larger or one smaller than that of $\AlgA$.
The recursive calls eventually
reach an instance covering one cell, which returns the sign of its
bias. This answer is then returned through the preceding calls to
$\AlgA_0$. After a prescribed number of continuation requests,
$\AlgB$ returns $\perp$ on the next request, and $\AlgA$ starts a
new epoch to answer that request.

\begin{algorithm}[tbp]
  \caption{Labeling strategy across epochs (\AlgA)}
  \label{alg:A}
  \begin{algorithmic}[1]
    \Require Parameters $l \in \mathbb{N}$, $r \in \mathbb{N}$, $l \leq r$,
    $b \in \mathbb{Z}$, with $r - l + 1$ a power of two.
    ($[l,r]$ is the covered interval and $b$ is the bias.)

    \Function{\AlgA.\initialize}{$l,r,b$}
      \State $\cnt \gets 0$.
      \If{$l < r$}
        \State $\currentB \gets \AlgB.\initialize(l,r,b,0)$.
      \EndIf
    \EndFunction

    \Function{\AlgA.\placeSign}{$s$}\Comment{$s \in [l,r]$}
      \If{$l = r$}
        \State $\sigma \gets \operatorname{sign}(b)$.
      \Else
        \State $\sigma \gets \currentB.\placeSign(s)$.
        \If{$\sigma = \perp$}
          \State $\currentB \gets
          \AlgB.\initialize(l,r,b,\cnt)$.\label{line:b-reinit}
          \State $\sigma \gets \currentB.\placeSign(s)$.
        \EndIf
      \EndIf
      \State $\cnt \gets \cnt + 1$.
      \State \Return $\sigma$.
    \EndFunction
  \end{algorithmic}
\end{algorithm}

\begin{algorithm}[tbp]
  \caption{Two-phase labeling epoch (\AlgB)}
  \label{alg:B}
  \begin{algorithmic}[1]
    \Require Parameters $l \in \mathbb{N}$, $r \in \mathbb{N}$, $l < r$,
    $b \in \mathbb{Z}$, $\tau\geq0$ an integer, with $r - l + 1$ a power of two.
    ($\tau$ is the total number of requests answered in all earlier epochs
    of the parent $\AlgA$ instance.)
    \Function{\AlgB.\initialize}{$l,r,b,\tau$}
      \State $m \gets \left\lfloor\frac{l + r}{2}\right\rfloor$.
      \State $\AlgA[0] \gets \AlgA.\initialize(l,m,b)$.
      \State $\AlgA[1] \gets \AlgA.\initialize(m + 1,r,b)$.
      \State $\cntHalf[0] \gets 0$, $\cntHalf[1] \gets 0$, $\phase \gets 1$.
    \EndFunction

    \Function{\AlgB.\placeSign}{$s$}
      \If{$\phase = 1$}
        \State $\half \gets 0$ if $s \leq m$ else $\half \gets 1$.
        \State $\cntHalf[\half] \gets \cntHalf[\half] + 1$.
        \State $\sigma \gets \AlgA[\half].\placeSign(s)$.
        \label{line:old-child-request}
        \State $t_h \gets \max\{\cntHalf[0],\cntHalf[1]\}$.
        \State $t_l \gets \min\{\cntHalf[0],\cntHalf[1]\}$.
        \If{$\delta t_l \geq \tau$ and $\rho t_h \leq t_l$}
          \label{line:trigger-test}
          \State $L \gets \lfloor\kappa t_h\rfloor$, $t_c\gets0$.
          \If{$L > 0$}
            \If{$\half = 0$}
              \State $\shiftedA \gets \AlgA.\initialize(l,r,b + 1)$.
            \Else
              \State $\shiftedA \gets \AlgA.\initialize(l,r,b - 1)$.
            \EndIf
          \EndIf
          \State $\phase\gets2$.
        \EndIf
        \State \Return $\sigma$.
        \label{line:trigger-old-answer}

      \ElsIf{$\phase = 2$}
        \If{$t_c = L$}\Comment{The continuation is complete.}
          \State \Return $\perp$.
          \label{line:countdown-termination}
        \EndIf
        \State $\sigma \gets \shiftedA.\placeSign(s)$.
        \label{line:continuation-request}
        \State $t_c \gets t_c + 1$.
        \State \Return $\sigma$.

      \EndIf
    \EndFunction
  \end{algorithmic}
\end{algorithm}

\paragraph{Algorithm parameters.}
The algorithms in \Cref{alg:A,alg:B} depend on three parameters
$\rho \in (0,1]$, $\kappa \in (0,1]$ and $\delta > 0$.
The parameter $\rho$ determines how balanced
the two counts must be before $\AlgB.\phase$ changes to 2, and
$\kappa$ determines the number of requests answered in phase~2.
The parameter $\delta$ scales the smaller count in the comparison with $\tau$.

\paragraph{Summary of \AlgA (\Cref{alg:A}).}
On input $(l,r,b)$, $\AlgA.\initialize$ initializes an instance of
$\AlgB$ with parameters $(l,r,b,0)$ if $l < r$,
and sets a counter for the number of rounds,
$\cnt = \executionSteps(\AlgA)$, to zero.
If $l = r$, then $\AlgA$ does not initialize a corresponding $\AlgB$ instance
and $\AlgA.\placeSign(s)$ chooses $\operatorname{sign}(b)$.
Otherwise, it calls $\currentB.\placeSign(s)$ to obtain a sign.
If this call returns $\perp$, meaning that the current instance of $\AlgB$ has
terminated, then $\AlgA$ initializes a new instance with parameters
$(l,r,b,\cnt)$ and passes the same request to it. After obtaining a
sign, $\AlgA$ increments $\cnt$ and returns the sign. Thus,
$\AlgA$ maintains a sequence of instances of $\AlgB$,
each of which answers requests until it returns $\perp$ or the game ends.
Note that at the end of every round involving a non-leaf instance,
\begin{equation}
  \executionSteps(\AlgA) = \tau + \executionSteps(\AlgA.\currentB).
  \label{eq:real-epoch-time}
\end{equation}
The variable $\cnt$ is the total number of requests processed so far across all
epochs. At a restart, $\cnt$ is exactly the number of requests answered in
earlier epochs, so the new epoch receives this number as $\tau$.
\paragraph{Summary of \AlgB (\Cref{alg:B}).}
On input $(l,r,b,\tau)$, $\AlgB.\initialize$ initializes two instances
of $\AlgA$, denoted by $\AlgA[0]$, $\AlgA[1]$, covering the left and
right halves of $[l,r]$, respectively, both with bias $b$. It also
maintains the two phase~1 request counts and denotes their minimum by
$t_l$ and their maximum by $t_h$.
In phase~1, each request is answered by the child covering its cell.
After obtaining this answer, $\AlgB$ checks whether the smaller count
$t_l$ is at least $\tau/\delta$ and at least $\rho t_h$, where $t_h$ is the
larger count. When both conditions hold, a \emph{trigger} occurs. It sets
$L = \lfloor\kappa t_h\rfloor$ and enters phase~2 for $L$ requests. If $L > 0$,
after obtaining the triggering request's answer from the old child, it
initializes $\shiftedA$ on the full interval $[l,r]$, with bias $b + 1$
if the trigger was in the left half, and $b - 1$
otherwise. The next $L$ requests, which are the \emph{continuation} of $\AlgB$,
are answered by $\shiftedA$, after which
$\AlgB$ returns $\perp$ on the following request. If $L = 0$, it returns
$\perp$ on the first request after entering phase~2. Note that the game
may end during either phase, in which case the instance need never
return $\perp$.

\subsection{Parameters and the labeling lemma}

\paragraph{Analysis parameters.}
The exponents $\alpha > 0$ and $\beta \in (0,1)$ describe how the number of signs remaining depends on space and time, respectively. The constant $\lambda \in (1,2]$
controls how changing bias affects the bounds, while $0 < K_p < 1$ and $K_t > K_p$ control,
respectively, the numbers of persistent and total inherited signs.
The constant $C > 0$ in the bound will be chosen in the proof.
We write $t = \executionSteps(\AlgA)$ and take $n$ to be a power
of two, so that each split gives two intervals of equal size. Our
goal is to choose the parameters with $\alpha + \beta$ as small as possible.

For an instance of $\AlgA$ with bias $b$ covering $n\geq2$ cells,
the coefficient in the bound for signs of type $\sigma \in \{-1,1\}$
for each of the two children of a $\AlgB$ instance is
$C\lambda^{b\sigma}\left(\frac{n}{2}\right)^\alpha$.
When phase~2 begins, the new instance of $\AlgA$ covers $n$ cells and
has bias $b + 1$ or $b - 1$. Relative to
$C\lambda^{b\sigma}\left(\frac{n}{2}\right)^\alpha$ for the original children,
its
contribution is multiplied by one of
  $\lambda^{-1}2^\alpha$ or $\lambda2^\alpha$. We collect all the inequalities needed by the analysis in the following
definition, and only later give values satisfying these inequalities in
\Cref{prop:explicit-parameters}.

\begin{definition}
  \label{def:real-admissible}
  A tuple $(\alpha,\beta,\lambda,\rho,\kappa,\delta)$ is \emph{admissible} if
  \begin{equation}
    \alpha > 0,\qquad 0 < \beta < 1,\qquad 1 < \lambda\leq2,\qquad
    0 < \rho\leq1,\qquad 0 < \kappa\leq1,\qquad \delta > 0,
    \label{eq:real-parameter-ranges}
  \end{equation}
  and there exist real constants $F$, $K_p$, $K_t$, $M$, $D$, $P$, $Q$
  satisfying all of the following conditions.

  Here, $F$ is the maximum number of rounds handled by the base case for the
  induction on time. $K_p$, $K_t$ are coefficients that bound
  inherited persistent signs and total signs, respectively.
  $M$ bounds epochs in phase~1, $D$ bounds epochs in all phases, and $P$,
  $Q$ control persistent
  and total signs at the end of epochs, respectively. The tuple is admissible if:

  \begin{enumerate}
    \item The constants lie in the ranges
    \begin{equation}
      \begin{gathered}
        F\geq1,\qquad
        0 < K_p < 1,\qquad K_t > K_p,\\
        K_p \leq M \leq D,\qquad P > 0,\qquad Q > 0.
      \end{gathered}
      \label{eq:real-auxiliary-ranges}
    \end{equation}

    \item To close the induction for short executions, persistent signs, and total signs, we have the bounds
    \begin{equation}
      \begin{gathered}
        P(1 + F^{-1}) + F^{-\beta} < K_p,\qquad Q(1 + F^{-1}) + F^{-\beta} <
        K_t,\\
        D + F^{-\beta} < 2^\alpha.
      \end{gathered}
      \label{eq:real-induction-slack}
    \end{equation}

    \item If an execution stops during phase 1, the inherited signs and the two half-interval instances are bounded as follows for all positive $x$ and $y$:
    \begin{equation}
      K_t x^\beta + y^\beta
      \leq D(x + y)^\beta.
      \label{eq:real-one-half-condition}
    \end{equation}
    If it stops during the continuation in phase 2, the inherited signs and the continuing instance are bounded as follows for all positive $x$ and $y$:
    \begin{equation}
        Mx^\beta + \lambda^{-1}2^\alpha y^\beta
       \leq D(x + y)^\beta.
      \label{eq:real-continuation-condition}
    \end{equation}

    \item If an instance is in phase~1, the inherited persistent signs and the signs from both halves are bounded. Specifically, for all nonnegative integers $\tau$, $t_h$, $t_l$ with
    $t_l \leq t_h$, if $\delta t_l < \tau$ or $t_l < \rho t_h$, then 
    \begin{equation}
      K_p\tau^\beta + t_h^\beta + t_l^\beta
      \leq M(\tau + t_h + t_l)^\beta.
      \label{eq:real-waiting-condition}
    \end{equation}
    If an instance has just switched into phase~2, the same bounds hold with an additional 1. Specifically, if instead $1 \leq t_l \leq t_h$, $\delta t_l \geq \tau$,
    $t_l \geq \rho t_h$, and either
    $\delta(t_l - 1) < \tau$ or $t_l - 1 < \rho t_h$, then
    \begin{equation}
      K_p\tau^\beta + t_h^\beta + t_l^\beta
      \leq M(\tau + t_h + t_l)^\beta + 1.
      \label{eq:real-trigger-condition}
    \end{equation}

    \item If an instance is triggered into phase~2, then $(z, r) \in \mathcal{T}$, where $z = \frac{\tau}{t_h}$ and $r = \frac{\max \{\tau/\delta, \rho t_h \}}{t_h}$ for integers $1 \leq t_l \leq t_h$, and 
    \begin{equation}
      \mathcal T := \{(z,r):0 \leq z \leq \delta r,\ \rho \leq r\leq1,
      \ r = \rho\text{ or }z = \delta r\}.
      \label{eq:real-trigger-domain}
    \end{equation}
    Then the numbers of persistent pluses, persistent minuses, and total signs, respectively, are bounded by
    \begin{align}
      K_pz^\beta + 1 + \lambda^{-1}2^\alpha\kappa^\beta
      & \leq P(z + 1 + r + \kappa)^\beta,
      \label{eq:real-normalized-p-plus}\\
      K_pz^\beta + r^\beta + \lambda2^\alpha\kappa^\beta
      & \leq P(z + 1 + r + \kappa)^\beta,
      \label{eq:real-normalized-p-minus}\\
      K_pz^\beta + 1 + r^\beta + \lambda^{-1}2^\alpha\kappa^\beta
      & \leq Q(z + 1 + r + \kappa)^\beta.
      \label{eq:real-normalized-q}
    \end{align}

    \item When a continuation finishes and an epoch ends, the numbers of persistent minuses, persistent pluses, and total signs are bounded. Specifically, for all nonnegative integers $\tau$, $t_h$, $t_l$ satisfying
    the count conditions of \eqref{eq:real-trigger-condition}, set
    $L = \lfloor\kappa t_h\rfloor$ and $t = \tau + t_l + t_h + L$.
    Whenever $t \geq F$,
    \begin{align}
      K_p\tau^\beta + t_l^\beta + \lambda2^\alpha L^\beta
      &< K_pt^\beta,
      \label{eq:real-terminal-p-minus}\\
      K_p\tau^\beta + t_h^\beta + \lambda^{-1}2^\alpha L^\beta
      &< K_pt^\beta,
      \label{eq:real-terminal-p-plus}\\
      K_p\tau^\beta + t_h^\beta + t_l^\beta
      +\lambda^{-1}2^\alpha L^\beta
      &< K_tt^\beta.
      \label{eq:real-terminal-q}
    \end{align}
    At any point during the continuation, the persistent count of the lesser sign remains within its induction bound, so for every real $0 \leq t_c \leq L$,
    \begin{equation}
      K_p\tau^\beta + t_l^\beta + \lambda2^\alpha t_c^\beta
      \leq K_p(\tau + t_l + t_h + t_c)^\beta.
      \label{eq:real-intermediate-condition}
    \end{equation}
  \end{enumerate}
\end{definition}

\begin{lemma}
  \label{lem:real-labeling}
  Fix an admissible tuple $(\alpha,\beta,\lambda,\rho,\kappa,\delta)$
  as in \Cref{def:real-admissible}.
  There exists a constant $C > 0$, depending only on this tuple, such that
  every instance $\AlgA$ with bias $b \in \mathbb Z$, covering a
  power-of-two number $n$ of cells, satisfies the following bound. For
  $t := \executionSteps(\AlgA)$ and each $\sigma \in \{-1,1\}$,
  \begin{equation}
    \remainingSigns(\AlgA,\sigma)
    \leq C\lambda^{b\sigma}n^\alpha t^\beta.
    \label{eq:real-inductive-bound}
  \end{equation}
\end{lemma}

\paragraph{Proof overview of \Cref{lem:real-labeling}.}
\label{par:real-inductive-setup}
Fix any admissible tuple
$(\alpha,\beta,\lambda,\rho,\kappa,\delta)$ and constants
$F$, $K_p$, $K_t$, $M$, $D$, $P$, $Q$ from \Cref{def:real-admissible}. Take $C \geq F^2/K_p$. We prove the bound
\eqref{eq:real-inductive-bound} in \Cref{lem:real-labeling} by an
outer induction on cells $n$, where $n$ is a power of two, and an inner
strong induction on requests $t$.
For an instance $\mathcal I$ of $\AlgA$ with bias $b'$, covering $n'$ cells and
having answered $t'$ requests, the \emph{space hypothesis} is, for $n' < n$,
\begin{equation}
  \remainingSigns(\mathcal I,\sigma)
  \leq C\lambda^{b'\sigma}(n')^\alpha(t')^\beta.
  \label{eq:overview-space-hypothesis}
\end{equation}
At the fixed space $n$, the \emph{time hypothesis} is, for $t' < t$,
\begin{equation}
  \remainingSigns(\mathcal I,\sigma)
  \leq C\lambda^{b'\sigma}n^\alpha(t')^\beta.
  \label{eq:overview-time-hypothesis}
\end{equation}
Both hypotheses use the same constant $C$ and hold
for all biases $b' \in \mathbb Z$ and sign types $\sigma \in \{-1,+1\}$.
Assuming both hypotheses, we fix a parent
instance $\AlgA$ with bias $b$ on $n\geq2$ cells and $t$ requests, noting that the base cases $t=0$ and $n=1$ are immediate.

\begin{itemize}
    \item  If $t\leq F$, we show that $C$ is large enough to establish the inductive claim by \Cref{lem:real-short-executions}. Thus, it suffices to only consider $t>F$.
  \item If only one half has been visited in the current epoch, the epoch is still in phase~1,
  and both persistent and non-persistent inherited signs may remain.
  We bound these with the total sign coefficient $K_t$ by \Cref{lem:real-inherited-signs}, giving 
  $K_t\tau^\beta+t_h^\beta\leq Dt^\beta$.
  \item Once both halves have been visited in the current epoch, the only inherited signs that remain are persistent, so we can sharpen the bound on inherited signs to the smaller persistent sign coefficient $K_p$
  by \Cref{lem:real-inherited-signs}. If the epoch is still in phase~1, then $t_l<\rho t_h$ or $\delta t_l<\tau$, so admissibility gives
  $K_p\tau^\beta+t_h^\beta+t_l^\beta\leq Mt^\beta$.
  \item If the epoch has entered phase~2, we account for the triggering
  deletion and combine the contributions of both the child instances and the shifted bias instance from the continuation, as in \Cref{lem:real-continuation-counts}. Using the inherited
  bounds from \Cref{lem:real-inherited-signs} gives coefficients $Dt^\beta+1$ and
  $K_pt^\beta$ for pluses and minuses, respectively, when the trigger is on the right; a left-half trigger exchanges the sign types.
  \item Since $K_p\leq M\leq D$, all three cases give
  $C\lambda^{b\sigma}(n/2)^\alpha(Dt^\beta+1)$, which is sufficient since the given inequality $D+F^{-\beta}<2^\alpha$ closes the induction.
  \end{itemize}

\subsection{Lemmas in the proof of \texorpdfstring{\Cref{lem:real-labeling}}{Lemma \getrefnumber{lem:real-labeling}}}

We begin by bounding the number of total signs and persistent signs in a continuation instance.
\label{subsec:epoch-bounds}
\begin{lemma}
  \label{lem:real-continuation-counts}
  Suppose an epoch of $\AlgA$, with instance $\AlgB$ and parameter
  $\tau$, starts with at most
  $C\lambda^{b\sigma}(n/2)^\alpha K_p\tau^\beta$ inherited persistent
  signs of each type, where $K_p$ is the inherited persistent-sign
  coefficient from \Cref{def:real-admissible}.
  Suppose $\AlgB.\phase = 2$ is triggered by the right half.
  Let $t_h$, $t_l$ be the larger and smaller counts at the trigger,
  and set $L = \lfloor\kappa t_h\rfloor$.
  During phase~2, after $0 \leq t_c \leq L$ continuation requests, we have
  \begin{align*}
    \remainingSigns(\AlgA, -1)& \leq C \lambda^{-b} \left( \frac{n}{2}
    \right)^\alpha \left( K_p \tau^\beta + t_l^\beta + \lambda2^\alpha t_c^\beta
    \right), \\
    \persistentSigns(\AlgA, +1)& \leq C \lambda^{b} \left( \frac{n}{2}
    \right)^\alpha \left( K_p \tau^\beta + t_h^\beta + \lambda^{-1}2^\alpha
    t_c^\beta \right), \\
    \remainingSigns(\AlgA, +1)& \leq C \lambda^{b} \left( \frac{n}{2}
    \right)^\alpha \left(K_p \tau^\beta + t_h^\beta + t_l^\beta +
    \lambda^{-1}2^\alpha t_c^\beta \right).
  \end{align*}
  For a left-half trigger, swap the two sign types and exchange the factors
  $\lambda^b$ and $\lambda^{-b}$.
\end{lemma}
\begin{proof}
  $\AlgB.\phase = 2$ implies that both halves have been visited. In particular,
  the first visit to the left half deletes all right-half pluses from earlier
  epochs, and the first visit to the right half deletes all earlier left-half
  minuses. Thus, only inherited persistent signs can remain from earlier epochs.

  Notice that the right count cannot be larger than the left count immediately
  after the phase changes. Otherwise, it was already at least the left count
  before this request, so the smaller count was unchanged and the larger count
  could only increase. Neither condition for changing the phase could therefore
  become true for the first time. Thus, $t_h$ denotes the count on the left and
  $t_l$ denotes the count on the right.

  By the inductive hypothesis on space, the signs of type $\sigma$ placed by
  $\AlgB.\AlgA[0]$ and $\AlgB.\AlgA[1]$ that remain are at most
  $C\lambda^{b\sigma}\left(\frac{n}{2}\right)^\alpha t_h^\beta$ and
  $C\lambda^{b\sigma}\left(\frac{n}{2}\right)^\alpha t_l^\beta$, respectively,
  because both these instances cover $\frac{n}{2} < n$ cells.
  After $\AlgB.\phase$ changes to 2, neither instance receives further requests,
  so these contributions cannot increase.

  If $t_c > 0$, then $L > 0$ and $\AlgB.\shiftedA$ has been initialized.
  Since the triggering request is on the right, this instance has
  bias $b - 1$ and covers $n$ cells. The parent $\AlgA$ has
  processed at least $t_h + t_l + t_c > t_c$ requests, so the inductive
  hypothesis on time applies to $\AlgB.\shiftedA$. Its plus and minus counts
  are therefore at most
  \begin{align*}
    C\lambda^{b - 1}n^\alpha t_c^\beta
    &= C \lambda^{b} \left( \frac{n}{2} \right)^\alpha \lambda^{-1}2^\alpha
    t_c^\beta,\\
    C\lambda^{-(b - 1)}n^\alpha t_c^\beta
    &= C\lambda^{-b}\left(\frac{n}{2}\right)^\alpha\lambda2^\alpha t_c^\beta.
  \end{align*}

  For $t_c = 0$, both contributions are zero, whether or not the
  continuation has been initialized.

  To bound the number of minuses, note that the right request on which the phase
  changes deletes every minus placed by $\AlgB.\AlgA[0]$.
  It suffices to consider the minuses placed by $\AlgB.\AlgA[1]$
  ($C\lambda^{-b}\allowbreak \left(\frac{n}{2}\right)^\alpha\allowbreak
  t_l^\beta$), the inherited persistent minuses ($C\lambda^{-b}\allowbreak
  \left(\frac{n}{2}\right)^\alpha\allowbreak  K_p \tau^\beta$),
  and the minuses placed by $\AlgB.\shiftedA$ ($C\lambda^{-b}\allowbreak
  \left(\frac{n}{2}\right)^\alpha\allowbreak \lambda2^\alpha t_c^\beta$).
  Summing these bounds gives the desired inequality on
  $\remainingSigns(\AlgA,-1)$.

  To bound the number of pluses, we first bound the number of persistent pluses.
  It suffices to consider the pluses placed by $\AlgB.\AlgA[0]$
  ($C\lambda^{b}\left(\frac{n}{2}\right)^\alpha t_h^\beta$),
  the inherited persistent pluses ($C\lambda^{b}\left(\frac{n}{2}\right)^\alpha
  K_p \tau^\beta$), and the pluses placed by $\AlgB.\shiftedA$
  ($C\lambda^{b}\left(\frac{n}{2}\right)^\alpha\lambda^{-1}2^\alpha t_c^\beta$).
  Counting all pluses from the last instance also bounds its persistent pluses.
  Summing these bounds gives the desired inequality on
  $\persistentSigns(\AlgA,+1)$. To obtain the inequality on
  $\remainingSigns(\AlgA,+1)$, we also include the pluses placed by
  $\AlgB.\AlgA[1]$ ($C\lambda^{b}\left(\frac{n}{2}\right)^\alpha t_l^\beta$).
\end{proof}

To induct on time, we first handle instances which have processed
at most $F$ requests. 

\begin{lemma}
  \label{lem:real-short-executions}
  Take the cutoff $F$ from \Cref{def:real-admissible} and
  $C \geq F^2/K_p$.
  Let $\AlgA$ be an instance with bias $b$ covering a power-of-two
  number $n$ of cells, and let $t = \executionSteps(\AlgA) \leq F$.
  Then, for each $\sigma \in \{-1,1\}$,
  \[
    \remainingSigns(\AlgA,\sigma)
    \leq C\lambda^{b\sigma}n^\alpha t^\beta.
  \]
\end{lemma}

\begin{proof}
  Since $K_p < 1$, our choice of $C$ gives $C \geq F^2$.
  If $t = 0$ or the instance has returned no sign of type $\sigma$,
  the desired bound is immediate. Otherwise, we first show that
  \[
    C\lambda^{b\sigma} \geq \frac{F^2}{t}.
  \]
  Fix a request on which the given instance returns $\sigma$. Write $\AlgA_0$,
  $\AlgA_1$, \dots, $\AlgA_k$ for the successive instances whose answers are
  passed back to produce this $\sigma$, where $\AlgA_0$ is the given instance
  and $\AlgA_k$ covers one cell. For each $i < k$,
  let $\AlgB_i$ be the instance called by $\AlgA_i$ that returns $\sigma$. Thus,
  $\AlgB_i$ obtains its answer from $\AlgA_{i + 1}$.
  Evaluate each instance's number of processed requests just after this request.

  If $\AlgB_i.\phase = 1$ when it calls $\AlgA_{i + 1}$,
  then $\AlgA_{i + 1}$ is one of $\AlgB_i.\AlgA[0]$ and $\AlgB_i.\AlgA[1]$.
  It has the same bias as $\AlgA_i$ and has processed at most as many requests.
  Otherwise, $\AlgB_i.\phase = 2$ and $\AlgA_{i + 1} = \AlgB_i.\shiftedA$,
  which was initialized with bias differing by one from that of $\AlgA_i$ when
  $\AlgB_i.\phase$ changed from 1 to 2. Let $t_h$,
  $t_l$ be the counts at that change, and let $t_c$ be the number of requests
  processed by $\AlgA_{i + 1}$. Then $t_c \leq L = \lfloor\kappa t_h\rfloor \leq
  t_h$, whereas $\AlgA_i$ has processed at least $t_h + t_l + t_c\geq2t_c$
  requests. Thus, each bias change reduces the number of processed requests by
  at least a factor of two. If there are $j$ bias changes among these instances,
  $\AlgA_k$ has processed at least one request and at most $t2^{-j}$ requests,
  so $j \leq \log_2t$.

  Let $b'$ be the bias of $\AlgA_k$. Since each bias change is by one,
  $|b' - b| \leq j$. The instance $\AlgA_k$ returns $\sigma =
  \operatorname{sign}(b')$, so $b'\sigma\geq0$. Consequently,
  $b\sigma \geq - j \geq - \log_2t$, and
  \[
    \lambda^{b\sigma} \geq \lambda^{-\log_2t} \geq 2^{-\log_2t} = \frac{1}{t}.
  \]
  The second inequality uses $\lambda\leq2$.
  Multiplying by $C \geq F^2$ gives the claimed inequality.

  Now, since $n^\alpha \geq 1$, $t^\beta \geq 1$ and $\frac{F^2}{t} \geq t$ for
  $t \leq F$, we have $C\lambda^{b\sigma}n^\alpha t^\beta \geq t$.
  This proves the desired bound, as the total number of signs preserved cannot
  be more than the total number of requests processed by the instance.
\end{proof}

\begin{lemma}
  \label{lem:real-inherited-signs}
  Let $K_p$ and $K_t$ be the inherited persistent- and total-sign coefficients.
  At the start of every epoch with parameter $\tau$, the following
  inequalities hold for each $\sigma \in \{-1,1\}$.
  \begin{align*}
    \persistentSigns(\AlgA,\sigma)
    & \leq C\lambda^{b\sigma}\left(\frac{n}{2}\right)^\alpha K_p\tau^\beta,\\
    \remainingSigns(\AlgA,\sigma)
    & \leq C\lambda^{b\sigma}\left(\frac{n}{2}\right)^\alpha K_t\tau^\beta.
  \end{align*}
\end{lemma}

\begin{proof}
  At the start of the first epoch, $\AlgA$ has placed no signs,
  so both bounds hold with $\tau = 0$. Suppose the bounds hold at the start of
  an epoch with parameter $\tau$, and that the next epoch begins. Let $t_h$,
  $t_l$ be the completed epoch's counts when $\AlgB.\phase$ changes from 1 to 2,
  and set $L = \lfloor\kappa t_h\rfloor$.
  The epoch answers $t_h + t_l + L$ requests before returning $\perp$ on the
  next request. Thus, the next epoch has parameter
  $u = \tau + t_h + t_l + L$, by \eqref{eq:real-epoch-time},
  and $\AlgA$ has answered exactly $u$ requests before that next epoch begins.
  Any deletions on the request that starts the new epoch can only reduce the
  inherited counts. Each continuation used to bound this completed epoch has
  fewer than $u \leq t$ requests, so the fixed time hypothesis applies.

  Suppose first that $u \leq F$. If $\AlgA$ has
  returned no sign of type $\sigma$, its count is zero.
  Otherwise, the argument in the proof of \Cref{lem:real-short-executions} gives
  $\lambda^{b\sigma} \geq \frac{1}{u}$. Since $n\geq2$ and
  $C \geq \frac{F^2}{K_p}$, we have
  \[
    C\lambda^{b\sigma}\left(\frac{n}{2}\right)^\alpha
    K_pu^\beta
    \geq \frac{CK_p}{u}
    \geq \frac{F^2}{u}
    \geq u.
  \]
  The total number of signs of type $\sigma$ placed by
  $\AlgA$ that remain is at most $u$. This proves both
  bounds for the next epoch, since persistent signs are
  a subset of all remaining signs and $K_p < K_t$.

  Suppose instead that $u > F$.
  First consider the case where the request on which
  $\AlgB.\phase$ changes from 1 to 2 is on the right.
  Applying \Cref{lem:real-continuation-counts} at $t_c = L$ and
  \eqref{eq:real-terminal-p-minus}, \eqref{eq:real-terminal-p-plus},
  and \eqref{eq:real-terminal-q} gives
  \begin{align*}
    \remainingSigns(\AlgA,-1)
    & \leq C \lambda^{-b} \left( \frac{n}{2} \right)^\alpha \left(K_p \tau^\beta
    + t_l^\beta + \lambda2^\alpha L^\beta \right)\\
    & \leq C\lambda^{-b}\left(\frac{n}{2}\right)^\alpha K_pu^\beta,\\
    \persistentSigns(\AlgA,+1)
    & \leq C \lambda^{b} \left( \frac{n}{2} \right)^\alpha \left(K_p \tau^\beta
    + t_h^\beta + \lambda^{-1}2^\alpha L^\beta \right)\\
    & \leq C\lambda^{b}\left(\frac{n}{2}\right)^\alpha K_pu^\beta,\\
    \remainingSigns(\AlgA,+1)
    & \leq C \lambda^{b} \left( \frac{n}{2} \right)^\alpha \left(K_p \tau^\beta
    + t_h^\beta + t_l^\beta
    +\lambda^{-1}2^\alpha L^\beta\right)\\
    & \leq C\lambda^{b}\left(\frac{n}{2}\right)^\alpha K_tu^\beta.
  \end{align*}
  Since $K_p < K_t$ and
  $\persistentSigns(\AlgA,-1)
  \leq \remainingSigns(\AlgA,-1)$,
  the required bounds for both sign types hold.
  If the trigger is on the left half instead,
  exchange the two sign types and the same arguments apply.

\end{proof}

\subsection{Completing the labeling induction}

It remains to combine the bounds above for an instance whose current
epoch need not have ended.

\begin{proof}[Proof of \Cref{lem:real-labeling}]
  Fix an admissible parameter tuple and constants
$F$, $K_p$, $K_t$, $M$, $D$, $P$, $Q$ from \Cref{def:real-admissible}. Take $C \geq F^2/K_p$. 
  We induct on space (the power of two $n$), and strongly on time ($t$) for each
  fixed $n$, uniformly over all biases and histories allowed in the lemma.

  For $t = 0$, the count is zero. For $n = 1$, at most one sign
  remains, and any returned type $\sigma$ satisfies $b\sigma\geq0$,
  so $C\lambda^{b\sigma}t^\beta\geq1$ when $t\geq1$.

  Now suppose $n\geq2$. \Cref{lem:real-short-executions} handles $t \leq F$,
  so assume
  $t > F$. The inductive hypothesis on space applies to both
  $\AlgB.\AlgA[0]$ and $\AlgB.\AlgA[1]$, since they cover fewer
  cells, and the inductive hypothesis on time applies to each $\AlgB.\shiftedA$,
  since it processes fewer
  requests. Thus, \Cref{lem:real-continuation-counts,lem:real-inherited-signs}
  apply.

  Let $\AlgB = \AlgA.\currentB$, with $t_c = 0$ in phase~1. By
  \eqref{eq:real-epoch-time}, we have $t = \tau + t_h + t_l + t_c$.
  There are three cases.

  \paragraph{Case 1: only one half of $\AlgB$ has been visited.}
  Then $\AlgB.\phase = 1$, and only the instance corresponding
  to the visited half has placed signs in this epoch.
  Both persistent and non-persistent signs from earlier epochs
  may remain. Here $t_l = 0$ and $t = \tau + t_h$. We have
  \begin{align*}
    \remainingSigns(\AlgA,\sigma)
    & \leq C \lambda^{b\sigma} \left( \frac{n}{2} \right)^\alpha \left(K_t
    \tau^\beta + t_h^\beta \right)\\
    & \leq C\lambda^{b\sigma}\left(\frac{n}{2}\right)^\alpha D t^\beta.
  \end{align*}
  The first inequality follows from \Cref{lem:real-inherited-signs} and the
  inductive hypothesis on space; the second follows from
  \eqref{eq:real-one-half-condition}.

  \paragraph{Case 2: both halves of $\AlgB$ have been visited,
  and $\AlgB.\phase = 1$.}
  Only persistent signs from earlier epochs remain.
  Since phase~1 has not ended, $\delta t_l < \tau$ or $t_l < \rho t_h$.
  With $t = \tau + t_h + t_l$, we have
  \begin{align*}
    \remainingSigns(\AlgA,\sigma)
    & \leq C \lambda^{b\sigma} \left( \frac{n}{2} \right)^\alpha \left(K_p
    \tau^\beta + t_h^\beta + t_l^\beta \right)\\
    & \leq C\lambda^{b\sigma}\left(\frac{n}{2}\right)^\alpha M t^\beta.
  \end{align*}
  The first inequality follows from \Cref{lem:real-inherited-signs}
  and the inductive hypothesis on space; the second follows from
  \eqref{eq:real-waiting-condition}.

  \paragraph{Case 3: $\AlgB.\phase = 2$.}
  Let $t_h$, $t_l$ be the counts when the phase changes from 1
  to 2, and let $t_c$ be the continuation count, with $t_c = 0$ if $L = 0$.
  Then $t = \tau + t_h + t_l + t_c$. Since $t_c \leq L$,
  \[
    \tau + t_h + t_l + L \geq t > F,
  \]
  so \eqref{eq:real-intermediate-condition} applies.

  Suppose the request on which the phase changes is on the
  right. For the plus signs, we have
  \begin{align*}
    \remainingSigns(\AlgA,+1)
    & \leq C \lambda^{b} \left( \frac{n}{2} \right)^\alpha \left(K_p \tau^\beta
    + t_h^\beta + t_l^\beta + \lambda^{-1}2^\alpha t_c^\beta \right)\\
    & \leq C \lambda^{b} \left( \frac{n}{2} \right)^\alpha \left(M(\tau + t_h +
    t_l)^\beta + \lambda^{-1}2^\alpha t_c^\beta + 1 \right)\\
    & \leq C\lambda^{b}\left(\frac{n}{2}\right)^\alpha\left(Dt^\beta + 1\right).
  \end{align*}
  The first inequality follows from
  \Cref{lem:real-continuation-counts,lem:real-inherited-signs},
  the second from \eqref{eq:real-trigger-condition}, and the third
  from \eqref{eq:real-continuation-condition}.
  For the minus signs, we have
  \begin{align*}
    \remainingSigns(\AlgA,-1)
    & \leq C \lambda^{-b} \left( \frac{n}{2} \right)^\alpha \left(K_p \tau^\beta
    + t_l^\beta + \lambda2^\alpha t_c^\beta \right)\\
    & \leq C\lambda^{-b}\left(\frac{n}{2}\right)^\alpha K_p t^\beta.
  \end{align*}
  The first inequality follows from
  \Cref{lem:real-continuation-counts,lem:real-inherited-signs};
  the second follows from \eqref{eq:real-intermediate-condition}.
  A left-half trigger exchanges the roles of the two sign types.

  To finish the proof, use $D \geq M \geq K_p$ from
  \eqref{eq:real-auxiliary-ranges}. Combining the
  three cases gives
  \[
    \remainingSigns(\AlgA,\sigma)
    \leq C \lambda^{b\sigma} \left( \frac{n}{2} \right)^\alpha \left(Dt^\beta +
    1 \right).
  \]
  Since $t > F$ and $0 < \beta < 1$,
  $Dt^\beta + 1
  = \left(D + t^{-\beta}\right)t^\beta
  \leq (D + F^{-\beta})t^\beta$.
  Therefore,
  \begin{align*}
    \remainingSigns(\AlgA,\sigma)
    & \leq C\lambda^{b\sigma}\left(\frac{n}{2}\right)^\alpha
    \left(Dt^\beta + 1\right)\\
    & \leq C\lambda^{b\sigma}n^\alpha t^\beta
    \frac{D + F^{-\beta}}{2^\alpha}\\
    & \leq C\lambda^{b\sigma}n^\alpha t^\beta,
  \end{align*}
  where the last inequality uses \eqref{eq:real-induction-slack}.
  This completes the inner induction on time and hence the outer induction on
  space.
\end{proof}

We now deduce \Cref{thm:real-root-bound} from \Cref{lem:real-labeling} by
summing the bounds for each sign type at the root and padding the board to
handle arbitrary sizes.

\begin{proof}[Proof of \Cref{thm:real-root-bound}]
  For a power-of-two board size, summing the bounds for each sign type in
  \Cref{lem:real-labeling} at the root, whose bias is $b = 0$, gives
  \[
    \opt(n,t) \leq C(\lambda^b + \lambda^{-b})n^\alpha t^\beta
    = 2Cn^\alpha t^\beta.
  \]
  For arbitrary $n$, pad
  the board to a power of two $N \in [n,2n)$ and request only the original $n$
  cells. Thus, we may take $K = 2^{1 + \alpha}C$.
\end{proof}

  \begin{samepage}
  \section{From sign-preservation bounds to calibration}
  \label{sec:optimizing-calibration}

  In this section, we prove the following theorem, converting separate
  space and time bounds for SPR into a calibration guarantee.

  \begin{theorem}
    \label{thm:intro-reduction}
    Let $\alpha \in (0,1)$, $\beta \in (0,1)$ and
    $\varepsilon := 1 - \alpha - \beta > 0$. Suppose
    $\opt(n,t) = O(n^\alpha t^\beta)$ for all integers $n\geq1$
    and $t\geq0$. Then, for every horizon $T\geq1$, there exists
    a randomized forecaster whose expected calibration error against
    any adversary satisfies
    \begin{equation*}
      \mathbb E[\calerr(T)]
      = O\!\left(
      T^{\frac{2 + \varepsilon}{3 + 2\varepsilon}}\right)
      = O\!\left(
      T^{\frac23 - \frac{\varepsilon}{9 + 6\varepsilon}}\right).
    \end{equation*}
  \end{theorem}
\end{samepage}

First, we establish the scale-sum
estimate in \Cref{lem:reduction-scale-sum}. Then,
we prove \Cref{thm:intro-reduction} by applying the SPR bound from
\Cref{sec:signflipping-ub}.

\subsection{Description of the strategy}

\paragraph{Conditional means.}
Our strategy uses the conditional mean approach developed by Hart
\cite{hart2023minimax} and Dagan et al.\ \cite{dagan2024}.
Fix any adversary and let
$e_t := \mathbb E[y_t\mid H_t]$, where $H_t$ is the history before
round $t$. For the construction, the forecaster observes $e_t$ before
predicting. We use deterministic labeling strategies and fixed tie-breaking
rules, so the
forecaster's prediction is a function of $H_t$. Specifically, all loops run in
increasing order, ties between grid points are resolved by choosing the
smallest, and ties between cells are resolved by choosing the leftmost.
After predicting, it observes $y_t$ and proceeds to the next round.

\paragraph{Notation.}
The forecasting strategy maintains several instances of the SPR game.
Each cell represents an interval of possible conditional means, and
the sign in that cell specifies an endpoint to predict: a plus selects
the left endpoint, and a minus selects the right endpoint. The
instances use intervals of different lengths and different thresholds
for the bias accumulated at their endpoints.

For a calibration horizon $T = 2^\tau$, choose integers $i_0$, $h$ with
$1 \leq i_0 < i_0 + h \leq \tau$. The parameter $i_0$ specifies the
coarsest partition, and $h$ is the number of partition levels. We
index these levels by
\begin{equation*}
  \mathcal I := \{i_0,\dots,i_0 + h - 1\}.
\end{equation*}
At each level $i$, a second index $j$ specifies the \emph{bias threshold}
$2^{j - i}$. The coarsest level $i = i_0$ uses $h$ thresholds, while each finer
level $i > i_0$ uses just one. Specifically, define
\begin{equation*}
  \mathcal J_i :=
  \begin{cases}
    \{i_0 + 1,\dots,i_0 + h\},&i = i_0,\\
    \{i_0 + h\},&i > i_0,
  \end{cases}
\end{equation*}
to represent the set of possible $j$ for each $i$.
At level $i$, partition $[0,1]$ into $2^{i + 1}$ equal intervals,
numbered from zero. For each $j \in \mathcal J_i$,
group the intervals by the parity of their indices and assign one instance to
each parity. The even and odd SPR instances contain
only cells representing the even-numbered and odd-numbered intervals,
respectively.
Thus any two intervals represented in the same game are separated by an interval
of the opposite parity. We write
$G_{i,j,\ell}$ for the instance with indices $i$,
$j$ and parity $\ell \in \{0,1\}$. The total number of SPR instances is
$2h + 2(h - 1) = 4h - 2 = O(\log T)$, compared with the $O((\log T)^2)$
instances used in the construction of Dagan et al.\ \cite{dagan2024}.

Note that an instance $G_{i, j, \ell}$ has $2^i$ cells.
Then for a cell $c \in [2^i]$, write
$\interval(c,G)$ for its interval and $\probub(c,\sigma,G)$ for
the prediction corresponding to the sign $\sigma \in \{-1,1\}$.
Setting $m = 2(c - 1) + \ell$, these are
\begin{equation*}
  \begin{split}
    \interval(c,G)& :=
    [m2^{-(i + 1)},(m + 1)2^{-(i + 1)}),\\
    \probub(c,\sigma,G)& :=
    \begin{cases}
      m2^{-(i + 1)},&\sigma = +1,\\
      (m + 1)2^{-(i + 1)},&\sigma = -1.
    \end{cases}
  \end{split}
\end{equation*}
The final interval also includes $1$. Note that all predictions lie in the fixed
grid
\begin{equation*}
  \mathcal P := \{k2^{-(i_0 + h + 1)}:k = 0,\dots,2^{i_0 + h + 1}\}.
\end{equation*}
Given a level $i$, a threshold index
$j$, and a conditional mean $e \in [0,1]$, the routine
$\cellub(i,j,e)$ locates the cell to which that mean belongs.
It returns the unique pair $(c,G)$ whose interval contains $e$. Specifically,
$m = \min\{\lfloor2^{i + 1}e\rfloor,2^{i + 1} - 1\}$,
$\ell = m\bmod2$, and $c = (m - \ell)/2 + 1$.

For each cell, the forecaster keeps a separate record of the bias
accumulated at each of its two endpoints. We denote these records by
$\biasub(c,\sigma,G)$ for $\sigma \in \{-1,1\}$, which both start at zero.
When a round is assigned to an endpoint, its bias is updated by adding
$e_t - p_t$. The value corresponding to the cell's
current sign is its \emph{active bias}, denoted $\biasub(c,G)$,
and the other value is \emph{inactive}. When the cell is empty, set
$\biasub(c,G) := 0$; both stored values are then inactive.

For every instance $G$, initialize the board of the game to be empty.
Fix a deterministic optimal base labeler for each instance
$G_{i,j,\ell}$ with space $2^i$ and time $r_{i,j,\ell}$, as defined in
\Cref{lem:reduction-round-count}.
The routine
$\simulateGame(c,G)$ supplies a sign for cell $c$ by simulating the
SPR game for instance $G$. When a new move is needed, it treats $c$
as the pointer's chosen cell and applies $G$'s labeling strategy.
The forecasting loop treats that strategy as a black box. In the
analysis below, we choose it by wrapping the base labeler with a
reduced transcript.

\begin{algorithm}[tbp]
  \raggedright

  \caption{Calibration strategy using sign preservation}
  \label{alg:reduction-forecast}
  \begin{algorithmic}[1]
    \Require Parameters $T = 2^\tau$, $i_0 \in \mathbb N$,
    $h \in \mathbb N$ with
    $1 \leq i_0 < i_0 + h \leq \tau$, and the deterministic
    \PlayerL strategies specified below.
    \State $\mathcal I \gets \{i_0,\dots,i_0 + h - 1\}$;
    $\mathcal J_{i_0} \gets \{i_0 + 1,\dots,i_0 + h\}$ and
    $\mathcal J_i \gets \{i_0 + h\}$ for $i > i_0$.
    \State Set $\mathcal P \gets
    \{k2^{-(i_0 + h + 1)}:k = 0,\dots,2^{i_0 + h + 1}\}$.
    \State Initialize each board $B_G$ to be empty.
    \State Initialize $\biasub(c, -1, G) = \biasub(c, 1,
    G) = 0$ for all instances $G$ and all cells $c$.
    \For{$t \in [T]$}
      \State Observe $e_t$.
      \For{$i \in \mathcal I$} \Comment{bias removal loop}
        \For{$j \in \mathcal J_i$}
          \For{$\ell \in \{0,1\}$}
            \State $G \gets G_{i,j,\ell}$.
            \If{$\exists c \in [2^i]:
              \bigl(\biasub(c,G) > 1\text{ and }
              e_t \leq \probub(c,+1,G)\bigr)
              \text{ or }
              \bigl(\biasub(c,G) < -1\text{ and }
              e_t \geq \probub(c,-1,G)\bigr)$}
              \label{line:reduction-removal}
              \State Choose the least such $c$.
              \State $\sigma \gets \operatorname{sign}(\biasub(c,G))$.
              \State Predict $p_t \gets \probub(c,\sigma,G)$. \Comment{predict to perform bias removal}
              \State $\biasub(c,\sigma,G) \gets
              \biasub(c,\sigma,G) + e_t - p_t$.
              \State Observe $y_t$ and end round $t$.
            \EndIf
          \EndFor
        \EndFor
      \EndFor
      \For{$i \in \mathcal I$} \Comment{bias placement loop}
        \For{$j \in \mathcal J_i$}
          \State $(c,G) \gets \cellub(i,j,e_t)$.
          \If{$|\biasub(c,G)| < 2^{j - i}$}
            \label{line:reduction-placement}
            \If{$c$ is empty in $B_G$}
              \State $\simulateGame(c,G)$. \Comment{place/delete signs in SPR}
            \EndIf
            \State $\sigma \gets$ sign in cell $c$ of $B_G$.
            \State Predict $p_t \gets \probub(c,\sigma,G)$. \Comment{predict to perform bias placement}
            \State $\biasub(c,\sigma,G) \gets
            \biasub(c,\sigma,G) + e_t - p_t$.
            \State Observe $y_t$ and end round $t$.
          \EndIf
        \EndFor
      \EndFor
      \State Predict a nearest point $p_t \in \mathcal P$ to $e_t$.
      \Comment{rounding fallback}
      \label{line:reduction-rounding}
      \State Observe $y_t$ and end round $t$.
    \EndFor
  \end{algorithmic}
\end{algorithm}

\paragraph{Summary of the forecaster (\Cref{alg:reduction-forecast}).}
Initially, all SPR boards are empty and all stored endpoint biases are
zero. At round $t$, after observing the conditional mean $e_t$, the
forecaster first checks every interval in the \emph{bias removal loop} to see if
bias can be removed. Specifically, it checks whether an endpoint prediction can
reduce an
active bias of magnitude greater than one. Its test at $(c,G)$ is
\begin{equation*}
  \begin{split}
    &\biasub(c,G) > 1
    \quad\text{and}\quad e_t \leq \probub(c,+1,G),\qquad\text{or}\\
    &\biasub(c,G) < -1
    \quad\text{and}\quad e_t \geq \probub(c,-1,G).
  \end{split}
\end{equation*}
If the test succeeds, the forecaster predicts the corresponding endpoint and
is able to remove some of the active bias from that cell.

If the test fails, the algorithm finds a cell for placement in the \emph{bias
placement loop} by looping through partition levels
$i \in \mathcal I$, starting from the coarsest level $i = i_0$ and ending at the
finest level $i = i_0 + h - 1$. For each $i$,
it loops through all thresholds in $\mathcal J_i$. For each $(i,j)$,
it locates the
cell $c$ whose interval contains $e_t$ and selects the first one with active
bias magnitude below $2^{j - i}$. If $c$ is empty,
we also call the routine $\simulateGame(c,G)$ and record the sign it returns.
If $\simulateGame(c, G)$ returns a plus or minus sign,
we predict the left or right endpoint of the interval of cell $c$, respectively,
and update the bias of the cell. We call a prediction and bias update at a cell
selected by this loop a \emph{placement} at that cell,
whether or not it was empty. If no such cell is selected by the bias removal or
bias placement loop, we fall back to simply predicting the nearest point of
$\mathcal P$
(Line~\ref{line:reduction-rounding}). Every prediction
is followed by observing $y_t$ and ending the round.

\subsection{The scale-sum lemma}

The following lemma bounds the calibration error of the conditional-mean
construction in terms of two sums of SPR guarantees. The sums correspond
to the coarsest level and the finer levels, respectively. We will apply
the bounds on SPR to these sums to prove
\Cref{thm:intro-reduction}.

\begin{lemma}
  \label{lem:reduction-scale-sum}
  Define
  \begin{align}
    S_0& :=
    \sum_{j = i_0 + 2}^{i_0 + h}2^{j - i_0}
    \opt(2^{i_0},2^{\tau - j + 1}),\notag\\
    S_1& :=
    \sum_{i = i_0 + 1}^{i_0 + h - 1}2^{i_0 + h - i}
    \opt(2^i,2^{\tau - i_0 - h}).\notag
  \end{align}
  By the end of round $T = 2^\tau$, the forecaster in
  \Cref{alg:reduction-forecast} satisfies
  \begin{equation}
    \mathbb E[\calerr(T)]
    = O\!\left(h2^{i_0} + 2^{\tau - i_0 - h} + S_0 + S_1
    +\sqrt{2^\tau(2^{i_0} + 2^{\tau - i_0 - h})}\right).
    \label{eq:calibration-scale-sum}
  \end{equation}
\end{lemma}

\paragraph{Proof overview of \Cref{lem:reduction-scale-sum}.}
We bound calibration error by conditional-mean bias and the
fluctuations $y_t-e_t$, then bound each contribution.
\begin{itemize}
  \item We first bound the conditional-mean bias in terms of the signs
  remaining in each game. An active endpoint has bias magnitude at most
  $2^{j-i}+1$, while an inactive endpoint has magnitude at most one by \Cref{lem:reduction-bias-invariants}. Thus a game with $S_G$ signs
  and $N_G$ used cells contributes at most $2^{j-i}S_G+2N_G$,
  and rounding adds $O(2^{\tau-i_0-h})$.
  \item To bound $S_G$, we undo wasted moves using the reduced transcript
  in \Cref{def:reduced-transcript}. Except at the smallest coarsest-level threshold, repeated recorded queries to a cell require
  bias to be rebuilt elsewhere. The resulting transcript bounds give
  $S_G\leq\opt(2^i,r_{i,j,\ell})$ with the round budgets from
  \Cref{lem:reduction-round-count}. Multiplying by the thresholds and
  summing over both parities bounds the bias from signs by $2S_0+2S_1+O(2^{i_0})$, where the last term accounts for the smallest coarsest-level threshold.
  \item We next account for the $N_G$ terms and the number of distinct
  predictions. Before a fine cell can be used, its parent must
  receive enough placements to reach its threshold. This shows that at most 
  $O(2^{\tau-i_0-h})$ fine cells are used, and there are at most
  $L=O(2^{i_0}+2^{\tau-i_0-h})$
  distinct predictions, including rounding, by \Cref{lem:reduction-used-intervals}.
  The coarsest-level cells contribute $O(h2^{i_0})$.
  \item Finally, the increments $y_t-e_t$ have conditional mean zero, so the
  prediction-wise sums have total second moment at most $T$. Cauchy--Schwarz and the bound
  on $L$ give an expected contribution of at most $\sqrt{LT}$.
  Adding this to the bias bound proves \eqref{eq:calibration-scale-sum}.
\end{itemize}
To deduce \Cref{thm:intro-reduction}, we apply the SPR power bound to
$S_0+S_1$ in \Cref{lem:reduction-scale-sum}, set $h=\tau-2i_0$,
and choose $i_0=\lfloor\tau/(3+2\varepsilon)\rfloor$ to balance the
leading bias and variance terms up to constant factors. We extend to arbitrary horizons
and apply minimax to remove access to the conditional means.

\subsection{Lemmas for the proof of \texorpdfstring{\Cref{lem:reduction-scale-sum}}{Lemma \getrefnumber{lem:reduction-scale-sum}}}

We first record the bias bounds and accounting identity,
adapting Lemma~A.1 of Dagan et al.\ \cite{dagan2024} to our endpoint predictions
and simulation
with restored signs.

\begin{samepage}
  \begin{lemma}
    \label{lem:reduction-bias-invariants}
    Let $T = 2^\tau$. Throughout \Cref{alg:reduction-forecast},
    the following hold.
    \begin{enumerate}
      \item For every cell $c$ in $G$,
      $\biasub(c,+1,G)\geq0$ and $\biasub(c,-1,G)\leq0$.
      \item For every cell $c$ in $G_{i,j,\ell}$ and sign type $\sigma \in \{-1,
      1 \}$, $|\biasub(c, \sigma, G)| \leq 2^{j - i} + 1$ if that value is
      active, and $|\biasub(c, \sigma, G)| \leq 1$ if that value is inactive.
      The magnitude of either positive or negative bias
      for a cell can only increase by at most $2^{-(i + 1)}$ per round.
      \item At the end of round $T$,
      \begin{equation}
        \sum_{p \in \mathcal P}\left|
        \sum_{t = 1}^T\mathbf1\{p_t = p\}(e_t - p)\right|
        \leq \sum_{G,c,\sigma}|\biasub_T(c,\sigma,G)|+O(2^{\tau - i_0 - h}).
        \label{eq:reduction-bias-accounting}
      \end{equation}
    \end{enumerate}
  \end{lemma}
\end{samepage}

\begin{proof}
  We use the induction underlying \cite[Lemma~A.1]{dagan2024}, keeping
  track of the two endpoint biases separately. Initially both are zero.
  A removal update moves an active value of magnitude greater than one
  toward zero by at most one, so it preserves its sign and cannot increase
  its magnitude.

  Before a call to $\simulateGame(c,G)$, every sign that may be deleted by $c$
  has bias of magnitude at most one, since a larger bias
  would have resulted in bias removal for that cell instead.
  Every sign deleted by this call therefore leaves
  behind a stored value of magnitude at most one.
  Since $c$ was empty before the call, both its biases have magnitude at most
  one. The returned sign therefore has
  active magnitude at most one, and this call preserves the
  inactive bound.

  At level $i$, a chosen cell's active bias has
  magnitude below $2^{j - i}$. Notice that by predicting at the endpoint
  corresponding to the chosen sign $\sigma$, we have
  \[
    0 \leq \sigma(e_t - p_t)\leq2^{-(i + 1)}.
  \]
  Thus a placement at a cell preserves the
  sign of the stored value and increases its magnitude by at most
  $2^{-(i + 1)}$, leaving it below $2^{j - i} + 1$.
  These observations complete the induction for
  items~1--2.

  For item~3, each stored value is the sum of the increments $e_t - p_t$
  assigned to its fixed endpoint, since it is never reset. Grouping
  these increments by prediction and applying the triangle inequality
  gives a bound of $\sum_{G,c,\sigma}|\biasub_T(c,\sigma,G)|$.
  Each rounding round adds
  at most half the grid spacing, so their total contribution is at most
  $T2^{-(i_0 + h + 2)} = O(2^{\tau - i_0 - h})$.
\end{proof}

Following Dagan et al.\ \cite[Appendix~A.2]{dagan2024},
we modify each instance's
labeler to remove wasted moves. If a plus is followed by a query to
its left, or a minus by a query to its right, that sign is deleted
immediately, so we can skip the move that placed it.
We formalize this intuition in the definition of a reduced transcript.

\begin{definition}[Reduced transcript]
  \label{def:reduced-transcript}
  Given a game $G$, call the forecaster's board the \emph{real board} and the
  base labeler's internal board the \emph{restored board}.
  Save the base labeler's full state before each move,
  including the restored board state. Suppose for a game $G$,
  the reduced transcript shows the pointer previously selected cells $\{c_1,
  c_2, \dots, c_{k - 1}\}$, from earliest to latest, and $\simulateGame(c_k,
  G)$ is called in the current round. First,
  simulate all sign deletions that are legally allowed by the call to cell $c_k$
  in the real board. If the transcript is nonempty and the sign placed in cell
  $c_{k - 1}$ can legally be deleted by this call,
  delete the previous call to $c_{k - 1}$ and restore the saved base-labeler
  state, including the restored board, from before that call.

  While the transcript is nonempty, continue to undo the last remaining call to
  $c_{k'}$ in reverse chronological order, from latest to earliest,
  while the call $c_k$ could legally delete the sign in cell $c_{k'}$.
  After these moves have been undone, if cell $c_k$ is occupied on the restored
  board, reuse its sign without querying the labeler or recording a new move;
  otherwise, query the labeler for a sign at $c_k$ and record the move.
  Then place the sign at cell $c_k$ on the real board.
  The recorded calls that are not deleted and remain after this process form the \emph{reduced
  transcript} of $G$.
\end{definition}

Note that a reduced transcript only affects the signs present on the board,
but not any of the bias values associated with a cell.
Rewinding restores older signs only on the
restored board, so that board contains
every actual sign, with the same sign in every actually occupied cell,
throughout the call. Restoring the full labeler state makes the reduced
transcript a legal play of the fixed base labeler.
The sign count coming from the reduced transcript thus
bounds the actual sign count after each completed call.

The next lemma bounds the length of the reduced transcript in
\Cref{def:reduced-transcript}; undone moves do not contribute to
its length.

\begin{lemma}
  \label{lem:reduction-round-count}
  Through round $T = 2^\tau$, the number of recorded moves in the
  reduced transcript of $G = G_{i,j,\ell}$ is at most $r_{i,j,\ell}$,
  where
  \begin{equation*}
    r_{i,j,\ell} :=
    \begin{cases}
      2^\tau,&i = i_0,\ j = i_0 + 1,\\
      2^{\tau - j + 1},&i = i_0,\ j \geq i_0 + 2,\\
      2^{\tau - i_0 - h},&i > i_0,\ j = i_0 + h.
    \end{cases}
  \end{equation*}
  Thus, the number of signs
  on the actual board of $G_{i,j,\ell}$ is at most $\opt(2^i,r_{i,j,\ell})$.
\end{lemma}

\begin{proof}
  Fix an instance and its reduced transcript through round $T$.
  Consider two successive recorded queries of a cell $c$.
  By definition of a reduced transcript, the move immediately after the first
  query must not delete the sign, and therefore lies on its preserving side,
  while the move that actually
  deletes that sign lies on its
  deleting side. These moves lie on opposite sides of $c$ and occur between
  the two queries of
  $c$. We consider two cases depending on the
  value of $i$.

  \paragraph{Case 1: $i = i_0$.}
  If $j = i_0 + 1$, the reduced transcript has length at most $T = r_{i,j,\ell}$
  because there are at most $T$ rounds played. Now consider $j \geq i_0 + 2$.
  Whenever a move at
  $c$ is recorded in $G_{i_0,j,\ell}$,
  $|\biasub(c,G_{i_0,j - 1,\ell})|\geq2^{j - 1 - i_0}$.
  Otherwise, the bias placement loop in \Cref{alg:reduction-forecast} would
  have stopped at the earlier instance because the condition on
  Line~\ref{line:reduction-placement} would be true. We now show that
  at some round between successive queries of $c$, we have
  $|\biasub(c,G_{i_0,j - 1,\ell})|\leq1$.
  If the sign at $c$ in $G_{i_0,j - 1,\ell}$
  changes between these queries, the conclusion is immediate because the cell
  must first become empty, when its active bias is zero. Otherwise,
  this follows because
  $\interval(c,G_{i_0,j,\ell}) = \interval(c,G_{i_0,j - 1,\ell})$,
  and the two moves to the left and right of $c$ are only possible when
  $|\biasub(c,G_{i_0,j - 1,\ell})|\leq1$; otherwise, the condition on
  Line~\ref{line:reduction-removal} of \Cref{alg:reduction-forecast}
  would be true and bias removal would be performed on $c$ in
  $G_{i_0,j - 1,\ell}$.

  By \Cref{lem:reduction-bias-invariants},
  both stored values have magnitude at most one at that round.
  Thus, between successive queries of $c$,
  the stored value active at the next query
  must increase in magnitude from at most $1$ to
  $2^{j - 1 - i_0}$. Each round assigned to its endpoint can increase its
  magnitude by at most
  $2^{-(i_0 + 1)}$, so this takes at least
  \[
    2^{i_0 + 1}(2^{j - 1 - i_0} - 1) = 2^j - 2^{i_0 + 1}\geq2^{j - 1}
  \]
  rounds. The first query needs as many rounds because the biases
  start at zero.

  \paragraph{Case 2: $i > i_0$.}
  Let $c'$ denote the parent of $c$ in $G_{i - 1,j,\ell'}$. Similarly,
  whenever a move at $c$ is recorded in $G_{i,j,\ell}$,
  $|\biasub(c',G_{i - 1,j,\ell'})|\geq2^{i_0 + h - i + 1}$.
  Otherwise, the bias placement loop in \Cref{alg:reduction-forecast} would
  have stopped at the earlier instance because the condition on
  Line~\ref{line:reduction-placement} would be true. Similarly, we show
  that at some round between successive queries of $c$, we have
  $|\biasub(c',G_{i - 1,j,\ell'})|\leq1$.
  If the sign at $c'$ in $G_{i - 1,j,\ell'}$
  changes between these queries, the conclusion is immediate because the cell
  must first become empty, when its active bias is zero. Otherwise,
  this follows because the only cell
  in $G_{i,j,\ell}$ that lies in $\interval(c',G_{i - 1,j,\ell'})$ is
  $c$. The other cell belongs to a game of different parity. Therefore,
  the conditional means for the two moves lie to the left and right of the
  parent interval, respectively.
  This implies
  $|\biasub(c',G_{i - 1,j,\ell'})|\leq1$; otherwise, the condition on
  Line~\ref{line:reduction-removal} of \Cref{alg:reduction-forecast}
  would be true and bias removal would be performed on $c'$ in
  $G_{i - 1,j,\ell'}$.

  Again, both stored values have magnitude at most one at that round.
  Thus, between successive queries of $c$,
  the stored value active at the next query
  must increase in magnitude from at most $1$ to
  $2^{i_0 + h - i + 1}$. Each round assigned to its endpoint can increase its
  magnitude by at most
  $2^{-i}$, so this takes at least
  \[
    2^i(2^{i_0 + h - i + 1} - 1) = 2^{i_0 + h + 1} - 2^i\geq2^{i_0 + h}
  \]
  rounds. The first query needs as many rounds because the biases
  start at zero.

  Note that between successive queries of one cell and across distinct cells,
  the required rounds that build bias are disjoint: cells have distinct
  corresponding cells in Case~1 or distinct parents in Case~2,
  and each round updates at most one of them. Thus, since $T = 2^\tau$,
  the above cases give the
  claimed bounds on $r_{i,j,\ell}$.
\end{proof}

The next lemma bounds the number of predictions used.

\begin{samepage}
  \begin{lemma}
    \label{lem:reduction-used-intervals}
    Through round $T = 2^\tau$,
    the forecaster uses at most
    $2^{i_0 + 1} + 1 + 2^{\tau - i_0 - h - 1}$ distinct prediction values.
  \end{lemma}
\end{samepage}
\begin{proof}
  Let $R$ count the distinct parent intervals at levels
  $i_0$, \dots, $i_0 + h - 1$ for which either the forecaster selects one
  of their children and predicts its
  endpoint, or the
  conditional mean lies in a child interval at level $i_0 + h$ on a
  final rounding round. Fix such a parent at level $i - 1$, where
  $i > i_0$. The first prediction in a child cannot be a removal round,
  since this would require an earlier placement at that child.
  When one of its children first contributes to $R$, the
  forecaster has passed over the parent in the bias placement loop.
  Thus the test on Line~\ref{line:reduction-placement} fails for the
  parent's cell in its instance with index $i_0 + h$, and its active
  stored bias has magnitude at least $2^{i_0 + h - i + 1}$.

  An endpoint's stored bias starts at zero,
  and each prediction assigned to this parent's
  endpoint increases its magnitude by at most $2^{-i}$. Hence reaching
  the threshold requires at
  least $2^{i_0 + h + 1}$ rounds. Each such round updates only one cell of one
  instance, so we can conclude
  \[
    R \leq T/2^{i_0 + h + 1} = 2^{\tau - i_0 - h - 1}.
  \]
  Every parent counted by $R$ at a level greater than $i_0$ must first have been
  used to predict an endpoint to build its bias,
  so its own parent is also counted.
  Starting from the coarsest partition, each counted parent adds at most one new
  prediction value, the common endpoint of its two children.
  The coarsest partition has $2^{i_0 + 1} + 1$ endpoints, so there are at most
  $2^{i_0 + 1} + 1 + R\leq2^{i_0 + 1} + 1 + 2^{\tau - i_0 - h - 1}$ distinct
  predictions.
\end{proof}

\subsection{Proof of \texorpdfstring{\Cref{lem:reduction-scale-sum}}{Lemma \getrefnumber{lem:reduction-scale-sum}}}

We bound calibration error by considering two separate terms.
These are the stored conditional-mean
bias, or the bias term, and the martingale fluctuations,
or the variance term. The preceding lemmas control
the bias by establishing the length of each reduced transcript $r_{i, j,
\ell}$ and control the fluctuations by bounding the number of distinct predictions.

\begin{proof}[Proof of \Cref{lem:reduction-scale-sum}]
  For each prediction $p$, the identity
  $y_t - p = (e_t - p) - (e_t - y_t)$ splits its cumulative calibration error
  into two sums. The first, from $e_t - p$, is the bias term bounded by
  \eqref{eq:reduction-bias-accounting} in
  \Cref{lem:reduction-bias-invariants}; the second, from the centered
  differences $e_t - y_t$, is the variance term. The triangle inequality gives
  \[
    \calerr(T) \leq
    \sum_{p \in \mathcal P}\left|
    \sum_{t = 1}^T\mathbf1\{p_t = p\}(e_t - p)\right|
    +\sum_{p \in \mathcal P}\left|
    \sum_{t = 1}^T\mathbf1\{p_t = p\}(e_t - y_t)\right|.
  \]
  \paragraph{Bounding the bias.} Let $N_G$ be the number of cells of $G$ that
  received placements through round $T$, and let $S_G$ be its number of signs
  at that time. By \Cref{lem:reduction-bias-invariants}, an occupied cell
  contributes at most $2^{j - i} + 2$, and an empty used cell contributes at
  most two. Consequently,
  \[
    \sum_{c,\sigma}|\biasub_T(c,\sigma,G)|
    \leq2^{j - i}S_G + 2N_G.
  \]
  At level $i = i_0$, each parity instance has $2^{i_0}$ cells.
  For $j = i_0 + 1$, the bias threshold is two and $S_G \leq N_G$, so its two
  instances contribute $O(2^{i_0})$ bias. For each
  $j = i_0 + 2$, \dots, $i_0 + h$, the board inclusion above and
  \Cref{lem:reduction-round-count} give
  $S_G \leq \opt(2^{i_0},2^{\tau - j + 1})$ in each parity instance.
  The $2^{j - i_0}S_G$ terms from these instances sum to at most
  $2S_0$; their $2N_G$ terms contribute $O(h2^{i_0})$.

  For $i = i_0 + 1$, \dots, $i_0 + h - 1$, the threshold index is
  $j = i_0 + h$. The same inclusion and round-count bound give
  $S_G \leq \opt(2^i,2^{\tau - i_0 - h})$ in each parity instance.
  Multiplying by the threshold $2^{i_0 + h - i}$ and summing over these
  levels gives at most $2S_1$. The remaining
  $2N_G$ terms count used intervals at levels $i > i_0$.
  Each such interval is a child of a parent counted in
  \Cref{lem:reduction-used-intervals}, and each parent has at most two
  children. Thus the total number of these intervals is
  $O(2^{\tau - i_0 - h})$. Hence, all
  contributions from cells that have received placements are
  $O(h2^{i_0} + 2^{\tau - i_0 - h})$.
  This bounds the bias by $O(h2^{i_0} + 2^{\tau - i_0 - h} + S_0 + S_1)$.

  \paragraph{Bounding the variance.} Let
  \[
    M_p := \sum_{t = 1}^T\mathbf1\{p_t = p\}(e_t - y_t),\qquad p \in \mathcal P.
  \]
  Note that the variables $p_t$, $e_t$ are measurable with respect to the
  history
  before $y_t$ is drawn, and $\mathbb E[y_t\mid H_t] = e_t$.
  Thus, for each $p$, the increments of $M_p$ have conditional mean
  zero, and increments at different times have zero expected product.
  Consequently,
  \begin{align*}
    \mathbb E\sum_{p \in \mathcal P}M_p^2
    &= \sum_{t = 1}^T\sum_{p \in \mathcal P}
    \mathbb E\!\left[\mathbf1\{p_t = p\}(y_t - e_t)^2\right]\\
    &= \sum_{t = 1}^T\mathbb E[e_t(1 - e_t)] \leq T.
  \end{align*}
  By \Cref{lem:reduction-used-intervals}, at most
  $L := 2^{i_0 + 1} + 1 + 2^{\tau - i_0 - h - 1}$ distinct prediction values
  are used on every run. Cauchy--Schwarz on each run, followed by Jensen's
  inequality, gives
  \[
    \mathbb E\sum_{p \in \mathcal P}|M_p|
    \leq \sqrt{L\,\mathbb E\sum_{p \in \mathcal P}M_p^2}
    \leq \sqrt{LT}
    = O\!\left(\sqrt{2^\tau(2^{i_0} + 2^{\tau - i_0 - h})}\right).
  \]
  Combining the two terms proves the lemma.
\end{proof}

\subsection{Proof of \texorpdfstring{\Cref{thm:intro-reduction}}{Theorem \getrefnumber{thm:intro-reduction}}}
\label{sec:reduction-power-bounds}

We apply the separate space and time bound to the two scale sums,
then set the value of the coarsest level $i_0$ to balance the bias and variance
terms. We conclude the proof by removing the construction's
access to conditional means.

\begin{proof}[Proof of \Cref{thm:intro-reduction}]
  First let the construction's horizon be $T = 2^\tau$. Choose
  \[
    i_0 = \left\lfloor\frac{\tau}{3 + 2\varepsilon}\right\rfloor,
    \qquad h = \tau - 2i_0.
  \]
  For sufficiently large $\tau$, these parameters satisfy
  $1 \leq i_0 < i_0 + h \leq \tau$. Each SPR board at level $i = i_0$
  has $2^{i_0}$ cells. Each instance at a level $i > i_0$ has a
  round budget of $2^{i_0}$, since $2^{\tau - i_0 - h} = 2^{i_0}$.
  Recall the estimate from \Cref{lem:reduction-scale-sum}:
  \[
    \mathbb E[\calerr(T)]
    = O\!\left(h2^{i_0} + 2^{\tau - i_0 - h} + S_0 + S_1
    +\sqrt{2^\tau(2^{i_0} + 2^{\tau - i_0 - h})}\right).
  \]
  We now apply the hypothesis $\opt(n,t) = O(n^\alpha t^\beta)$ to express the
  two sums $S_0$ and $S_1$ in terms of the parameters.
  Write $\tau = 2i_0 + h$ and $\alpha + \beta = 1 - \varepsilon$.
  The exponents in the two sums can then be rewritten as
  \begin{align*}
    j - i_0 + \alpha i_0 + \beta(\tau - j + 1)
    &= \tau - (1 + \varepsilon)i_0 + \beta
    -(1 - \beta)(i_0 + h - j),\\
    i_0 + h - i + \alpha i + \beta(\tau - i_0 - h)
    &= \tau - (1 + \varepsilon)i_0
    -(1 - \alpha)(i - i_0).
  \end{align*}
  The common leading term produces the factor
  $2^{\tau - (1 + \varepsilon)i_0}$ in both bounds. The extra $2^\beta$
  in the first identity is a constant. Thus, for the terms at
  level $i = i_0$,
  \begin{align*}
    2^{j - i_0}\opt(2^{i_0},2^{\tau - j + 1})
    &= O\!\left(2^{j - i_0 + \alpha i_0 + \beta(\tau - j + 1)}\right)\\
    &= O\!\left(2^{\tau - (1 + \varepsilon)i_0}
    \cdot 2^{-(1 - \beta)(i_0 + h - j)}\right).
  \end{align*}
  The exponent $i_0 + h - j$ is nonnegative for every $j$ in this
  sum. Similarly, for the terms at levels $i > i_0$,
  \begin{align*}
    2^{i_0 + h - i}\opt(2^i,2^{\tau - i_0 - h})
    &= O\!\left(2^{i_0 + h - i + \alpha i + \beta(\tau - i_0 - h)}\right)\\
    &= O\!\left(2^{\tau - (1 + \varepsilon)i_0}
    \cdot 2^{-(1 - \alpha)(i - i_0)}\right).
  \end{align*}
  Here $i - i_0\geq1$. Expanding the definitions of $S_0$ and $S_1$
  and applying these bounds term by term yields
  \begin{align*}
    S_0 + S_1
    &= \sum_{j = i_0 + 2}^{i_0 + h}2^{j - i_0}
    \opt(2^{i_0},2^{\tau - j + 1})
    +\sum_{i = i_0 + 1}^{i_0 + h - 1}2^{i_0 + h - i}
    \opt(2^i,2^{\tau - i_0 - h})\\
    &= O\!\left(2^{\tau - (1 + \varepsilon)i_0} \cdot
    \left(\sum_{j = i_0 + 2}^{i_0 + h}2^{-(1 - \beta)(i_0 + h - j)}
    +\sum_{i = i_0 + 1}^{i_0 + h - 1}2^{-(1 - \alpha)(i - i_0)}\right)\right)\\
    &= O\!\left(2^{\tau - (1 + \varepsilon)i_0} \cdot
    \left(\sum_{q = 0}^{h - 2}2^{-(1 - \beta)q}
    +\sum_{q = 1}^{h - 1}2^{-(1 - \alpha)q}\right)\right)\\
    &= O\!\left(2^{\tau - (1 + \varepsilon)i_0}\right).
  \end{align*}
  In the penultimate line, set $q = i_0 + h - j$ in the first sum and
  $q = i - i_0$ in the second. Because $\alpha < 1$ and $\beta < 1$, both ratios
  $2^{-(1 - \beta)}$ and $2^{-(1 - \alpha)}$ are strictly below $1$;
  the finite sums are therefore bounded by
  $1/(1 - 2^{-(1 - \beta)})$ and $1/(1 - 2^{-(1 - \alpha)})$,
  respectively, independently of $h$. Empty sums are zero.
  Substituting $S_0 + S_1 = O(2^{\tau - (1 + \varepsilon)i_0})$ and
  $2^{\tau - i_0 - h} = 2^{i_0}$ into this estimate, and using $h\geq1$,
  gives
  \[
    \mathbb E[\calerr(T)]
    = O\!\left(
    2^{\tau - (1 + \varepsilon)i_0}
    +h2^{i_0} + 2^{(\tau + i_0)/2}\right)
  \]
  for the forecaster that uses $e_t = \mathbb E[y_t\mid H_t]$ before
  predicting. Our choice of $i_0$ balances
  the first and third terms up to constant factors, since
  \[
    1 - \frac{1 + \varepsilon}{3 + 2\varepsilon}
    = \frac12 + \frac{1}{2(3 + 2\varepsilon)}
    = \frac{2 + \varepsilon}{3 + 2\varepsilon}.
  \]
  The remaining term is smaller: $h \leq \tau$, and
  \[
    \frac{2 + \varepsilon}{3 + 2\varepsilon}
    -\frac{1}{3 + 2\varepsilon}
    = \frac{1 + \varepsilon}{3 + 2\varepsilon} > 0.
  \]
  Rounding $i_0$ changes only constant factors. Thus
  \[
    \mathbb E[\calerr(T)]
    = O\!\left(
    T^{\frac{2 + \varepsilon}{3 + 2\varepsilon}}\right)
    = O\!\left(
    T^{\frac23 - \frac{\varepsilon}{9 + 6\varepsilon}}\right).
  \]
  To generalize to an arbitrary sufficiently large horizon $T$, take
  $\tau = \lceil\log_2 T\rceil$ and run the construction with
  budget $2^\tau$ for its first $T$ rounds. The bias and variance
  arguments use only $2^\tau$ as an upper bound on the number of
  elapsed rounds, so they apply to this prefix. Since $2^\tau < 2T$,
  this changes only the implied constant in the bound.
  For the remaining small horizons, enlarge the implied constant and use
  $\calerr(T) \leq T$.

  We must now address the fact that so far the forecaster has used $e_t$ before
  predicting, but this is not allowed in the original game setting.
  It remains to remove the forecaster's access to $e_t$, which is unavailable in
  the original game.
  For each fixed randomized adversary, suppose that the forecaster knows the
  adversary's strategy. The forecaster can compute $e_t = \Pr(y_t = 1\mid H_t)$
  from the observed history and run the construction above,
  achieving the stated bound against that adversary. Thus,
  every randomized adversary admits a forecaster achieving the
  same bound. Since the fixed
  prediction grid $\mathcal P$ is finite, this
  allows us to apply finite-game minimax,
  as in Hart \cite{hart2023minimax} and Dagan et al.\ \cite{dagan2024},
  to obtain one randomized
  forecaster with the same bound against any adversary,
  without knowing the adversary's strategy and without
  access to $e_t$.
\end{proof}

  \section{Explicit parameters and the main theorem}
\label{sec:explicit-parameters}
\label{subsec:real-parameter-optimization}

The following parameter choice completes the proof of our calibration
bound. We defer the verification that these parameters satisfy the necessary
inequalities of \Cref{def:real-admissible} to \Cref{app:parameter-verification}.

\begin{proposition}
  \label{prop:explicit-parameters}
  The following six-parameter tuple is admissible:
  \begin{equation}
    \begin{gathered}
      (\alpha,\beta,\lambda,\rho,\kappa,\delta)\\
      = \bigl(0.0082540892,\ 0.9574603467,\ 1.9444858156,\\
      0.0164237243,\ 0.4148723339,\ 9.4880928733\bigr).
    \end{gathered}
    \label{eq:explicit-parameters}
  \end{equation}
  We may also choose the constants
  \begin{equation}
    F = 10^{10},\qquad K_p = 0.8812155678,\qquad K_t = 0.9208553424,
    \qquad C = F^2/K_p
    \label{eq:explicit-auxiliary-parameters}
  \end{equation}
  and define $M$, $D$, $P$, $Q$ as in
  \Cref{subsec:real-auxiliary-verification}. These values satisfy
  \eqref{eq:real-induction-slack}.
  Moreover, $0 < \alpha < \beta < 1$ and
  \[
    \varepsilon := 1 - \alpha - \beta = 0.0342855641 > 0.
  \]
\end{proposition}

\begin{proof}[Proof of \Cref{thm:intro-calibration}]
  Use the parameters and constant $C$ in \Cref{prop:explicit-parameters}.
  By \Cref{thm:real-root-bound}, the labeling
  strategy guarantees uniformly over integers $n\geq1$ and $t\geq0$ that for $K
  = 2^{1 + \alpha}C$,
  \[
    \opt(n,t) \leq K n^{0.0082540892}t^{0.9574603467}.
  \]
  Since $0 < \alpha < \beta < 1$ and $\varepsilon > 0$,
  \Cref{thm:intro-reduction} shows that there exists a forecaster whose expected
  calibration error against any adversary is
  \[
    \mathbb E[\calerr(T)] = O\!\left(T^\theta\right),
    \qquad
    \theta = \frac23 - \frac{\varepsilon}{9 + 6\varepsilon}.
  \]
  $\theta < 0.662942288$ by \Cref{lem:parameter-calibration-exponent},
  which yields the desired upper bound on calibration error.
\end{proof}

\begin{remark}
  The parameters presented in \Cref{prop:explicit-parameters} are close to the
  best values found in our numerical search. We minimize the calibration
  exponent by maximizing $\varepsilon = 1 - \alpha - \beta$ subject to the
  parameter inequalities. In Python, SciPy's
  \path{scipy.optimize.differential_evolution} first searches broadly for
  candidate parameters, penalizing parameters if they violate inequalities.
  We then refine the candidate and additional random starting points with
  \path{scipy.optimize.minimize},
  using sequential least squares programming to impose the inequalities
  as explicit constraints.
  The smallest calibration exponent found is approximately
  $0.66294228152554$, improving on our choice by only $6.3\cdot10^{-9}$.
\end{remark}

  \bibliography{references}

  \appendix
  \section{Verification of the explicit parameters}
\label{app:parameter-verification}

This appendix proves \Cref{prop:explicit-parameters} and verifies the resulting
calibration exponent. We give explicit choices of $M$, $D$, $P$, $Q$,
reduce admissibility to three strict inequalities, and check them using
Python with the \texttt{mpmath} library.

\subsection{Verification of the explicit bounds}
\label{subsec:real-auxiliary-verification}
\label{subsec:parameter-verification}

Fix
$\alpha$, $\beta$, $\lambda$, $\rho$, $\kappa$,
$\delta$  satisfying \eqref{eq:real-parameter-ranges}, but we do not yet assume
admissibility for this tuple in the proof. Choose constants $F \geq 1$, $0 < K_p < 1$,
and $K_t > K_p$. Set $x_* := \min\{\delta^{-1},K_p^{-1/(1 - \beta)}\}$ and
define the phase~1 coefficient
\begin{equation}
  \begin{split}
    M := \max\Biggl\{&
    \left[ \left( \frac{K_p + x_*^\beta}{(1 + x_*)^\beta} \right)^{ \frac{1}{1 -
    \beta}} + 1 \right]^{1 - \beta},
    \\
    &
    \left[K_p^{ \frac{1}{1 - \beta}} + \left( \frac{1 + \rho^\beta}{(1 +
    \rho)^\beta} \right)^{ \frac{1}{1 - \beta}} \right]^{1 - \beta}
    \Biggr\}.
  \end{split}
  \label{eq:real-phase-one-coefficient}
\end{equation}

The coefficients $P$, $Q$ control persistent and total signs,
respectively, at the end of phase~2 before rounding, with
$\mathcal T$ as in \eqref{eq:real-trigger-domain}:
\begin{align}
  P& := \sup_{(z,r) \in \mathcal T}
  \max\left\{
  \frac{K_p z^\beta + r^\beta + \lambda2^\alpha \kappa^\beta}{(z + 1 + r +
  \kappa)^\beta}, \frac{K_p z^\beta + 1 + \lambda^{-1}2^\alpha \kappa^\beta}{(z
  + 1 + r + \kappa)^\beta}
  \right\},\label{eq:real-p-coefficient}\\
  Q& := \sup_{(z,r) \in \mathcal T}
  \frac{K_p z^\beta + 1 + r^\beta + \lambda^{-1}2^\alpha \kappa^\beta}{(z + 1 +
  r + \kappa)^\beta}.
  \label{eq:real-q-coefficient}
\end{align}

Finally, $D$ combines the bounds for an execution that may stop in
either phase:
\begin{equation*}
  D := \max \left\{ \left(M^{ \frac{1}{1 - \beta}} + \left( \lambda^{-1}2^\alpha
  \right)^{ \frac{1}{1 - \beta}} \right)^{1 - \beta},
  \left(K_t^{ \frac{1}{1 - \beta}} + 1 \right)^{1 - \beta} \right\}.
\end{equation*}

\paragraph{Ranges and bounds for unfinished epochs.}
We begin by establishing a weighted concavity inequality.
\begin{lemma}
  \label{lem:real-weighted-concavity}
  For every $\beta \in (0,1)$ and all $A\geq0$, $B\geq0$, $x\geq0$, $y\geq0$,
  \begin{equation*}
    Ax^\beta + By^\beta
    \leq \left(A^{ \frac{1}{1 - \beta}} + B^{ \frac{1}{1 - \beta}} \right)^{1 -
    \beta}(x + y)^\beta.
  \end{equation*}
\end{lemma}

\begin{proof}
  If $A = 0$ or $B = 0$, the inequality follows from the monotonicity of
  $u \mapsto u^\beta$ on $[0,\infty)$. Otherwise, set
  \[
    S := A^{\frac{1}{1 - \beta}} + B^{\frac{1}{1 - \beta}},
    \qquad p := \frac{A^{\frac{1}{1 - \beta}}}{S} \in (0,1).
  \]
  Since $u \mapsto u^\beta$ is concave on $[0,\infty)$,
  \begin{align*}
    Ax^\beta + By^\beta
    &= S^{1 - \beta}\left[p\left(\frac{x}{p}\right)^\beta
    +(1 - p)\left(\frac{y}{1 - p}\right)^\beta\right]\\
    & \leq S^{1 - \beta}(x + y)^\beta. 
  \end{align*}
\end{proof}

The second term defining $M$ is greater than $K_p$.
The definition of $D$ gives
$D \geq M$. The functions defining $P$, $Q$ are positive and continuous
on the nonempty compact domain $\mathcal T$,
so their suprema are finite and positive. This verifies all the
remaining ranges in \eqref{eq:real-auxiliary-ranges}.
For $x\geq0$, $y\geq0$, the weighted concavity inequality
\Cref{lem:real-weighted-concavity} gives
\begin{align*}
  K_tx^\beta + y^\beta
  & \leq \left(K_t^{\frac{1}{1 - \beta}} + 1\right)^{1 - \beta}(x + y)^\beta
  \leq D(x + y)^\beta,\\
  Mx^\beta + \lambda^{-1}2^\alpha y^\beta
  & \leq \left(M^{\frac{1}{1 - \beta}}
  +(\lambda^{-1}2^\alpha)^{\frac{1}{1 - \beta}}\right)^{1 - \beta}
  (x + y)^\beta
  \leq D(x + y)^\beta.
\end{align*}
Thus \eqref{eq:real-one-half-condition} and
\eqref{eq:real-continuation-condition} hold.

\paragraph{Phase~1 inequalities.}
The following lemma verifies \eqref{eq:real-waiting-condition} and
\eqref{eq:real-trigger-condition} for the explicit choice of $M$.

\begin{lemma}
  \label{lem:real-waiting-counts}
  For the phase~1 coefficient $M$ in \eqref{eq:real-phase-one-coefficient},
  \eqref{eq:real-waiting-condition} and
  \eqref{eq:real-trigger-condition} hold under their respective
  count conditions. In particular, they apply when $\AlgB.\phase = 1$
  at the end of a round and when a request changes the phase from
  1 to 2, respectively.
\end{lemma}
\begin{proof}
  Suppose first that $\delta t_l < \tau$ or $t_l < \rho t_h$, as required in
  \eqref{eq:real-waiting-condition}.

  If $\delta t_l < \tau$, then $\tau > 0$ and $t_l/\tau < \delta^{-1}$.
  For $x > 0$, the derivative of $\frac{K_p + x^\beta}{(1 + x)^\beta}$ has the
  sign of $x^{\beta - 1} - K_p$, which is nonnegative exactly when $x \leq
  K_p^{-1/(1 - \beta)}$. By continuity, the ratio is maximized on
  $[0,\delta^{-1}]$ at $x_*$. Thus,
  $K_p \tau^\beta + t_l^\beta \leq \frac{K_p + x_*^\beta}{(1 + x_*)^\beta}(\tau
  + t_l)^\beta$, and by \Cref{lem:real-weighted-concavity},
  \[
    K_p \tau^\beta + t_l^\beta + t_h^\beta \leq \left[ \left( \frac{K_p +
    x_*^\beta}{(1 + x_*)^\beta} \right)^{ \frac{1}{1 - \beta}} + 1 \right]^{1 -
    \beta}(\tau + t_l + t_h)^\beta.
  \]
  Instead, if $\rho t_h > t_l$, then $t_h > 0$ and $\frac{t_l}{t_h} < \rho$.
  For $x \in (0,1]$, the derivative of $\frac{1 + x^\beta}{(1 + x)^\beta}$ has
  the sign of $x^{\beta - 1} - 1$, which is nonnegative. By continuity,
  this ratio is also nondecreasing on $[0,1]$. Thus,
  $t_l^\beta + t_h^\beta \leq \frac{1 + \rho^\beta}{(1 + \rho)^\beta}(t_l +
  t_h)^\beta$, and by \Cref{lem:real-weighted-concavity},
  \[
    K_p \tau^\beta + t_l^\beta + t_h^\beta \leq \left[K_p^{ \frac{1}{1 - \beta}}
    + \left( \frac{1 + \rho^\beta}{(1 + \rho)^\beta} \right)^{ \frac{1}{1 -
    \beta}} \right]^{1 - \beta}(\tau + t_l + t_h)^\beta.
  \]
  These two cases establish \eqref{eq:real-waiting-condition}.
  Under the conditions of \eqref{eq:real-trigger-condition},
  \eqref{eq:real-waiting-condition} applies to $(\tau,t_h,t_l - 1)$. Since
  $(x + 1)^\beta - x^\beta\leq1$ for all $x\geq0$ and $0 < \beta < 1$,
  \[
    K_p \tau^\beta + t_h^\beta + t_l^\beta \leq M(\tau + t_h + t_l - 1)^\beta +
    1 \leq M(\tau + t_h + t_l)^\beta + 1.
  \]
  For an actual epoch, a round ending in phase~1 fails the trigger
  test in line~\ref{line:trigger-test} of \Cref{alg:B}, so the first
  condition holds. At the first trigger, the requested half has a
  smaller count just before the request: increasing a strictly larger
  count cannot make either part of the trigger test become true, and
  increasing one of two equal counts cannot do so either. Thus the
  sorted counts change from $(t_h,t_l - 1)$ to $(t_h,t_l)$, which gives
  the second condition.
\end{proof}

\paragraph{Phase~2 inequalities.}
Each expression in the maximum defining $P$ is at most $P$, and
the expression defining $Q$ is at most $Q$. Multiplying by the
positive denominator $(z + 1 + r + \kappa)^\beta$ proves
\eqref{eq:real-normalized-p-plus}, \eqref{eq:real-normalized-p-minus},
and \eqref{eq:real-normalized-q}, respectively.
The next lemma verifies \eqref{eq:real-terminal-p-minus},
\eqref{eq:real-terminal-p-plus},
\eqref{eq:real-terminal-q}, and
\eqref{eq:real-intermediate-condition}, conditional on the first two
inequalities in \eqref{eq:real-induction-slack}.

\begin{lemma}
  \label{lem:real-terminal-bounds}
  Take the endpoint coefficients $P$, $Q$ from
  \eqref{eq:real-p-coefficient}--\eqref{eq:real-q-coefficient} and suppose
  $P(1 + F^{-1}) + F^{-\beta} < K_p$ and $Q(1 + F^{-1}) + F^{-\beta} < K_t$.
  Then \eqref{eq:real-terminal-p-minus}, \eqref{eq:real-terminal-p-plus},
  \eqref{eq:real-terminal-q}, and
  \eqref{eq:real-intermediate-condition} hold for all counts in
  the domain specified in \Cref{def:real-admissible}.
\end{lemma}
\begin{proof}
  Fix counts satisfying the conditions in \Cref{def:real-admissible},
  and set $L = \lfloor\kappa t_h\rfloor$ and $t = \tau + t_l + t_h + L \geq F$.
  Set $\widetilde t_l := \max\{\tau/\delta,\rho t_h\}$, $z = \tau/t_h$,
  and $r = \widetilde t_l/t_h$.
  The specified count conditions give $t_l - 1 < \widetilde t_l \leq t_l$
  and $(z,r) \in \mathcal T$.
  Thus,
  \begin{align*}
    K_p\tau^\beta + t_l^\beta + \lambda2^\alpha L^\beta
    & \leq K_p\tau^\beta + \widetilde t_l^\beta + \lambda 2^\alpha (\kappa
    t_h)^\beta + 1 \\
    &= t_h^\beta(K_pz^\beta + r^\beta + \lambda 2^\alpha \kappa^\beta) + 1\\
    & \leq P(\tau + \widetilde t_l + t_h + \kappa t_h)^\beta + 1\\
    & \leq P(t + 1)^\beta + 1 \\
    &= Pt^\beta(1 + t^{-1})^\beta + 1 \\
    & \leq \bigl(P(1 + F^{-1}) + F^{-\beta}\bigr)t^\beta \\
    &< K_pt^\beta.
  \end{align*}
  The second inequality follows from the definition of $P$,
  the third inequality follows because $\kappa t_h < L + 1$ implies $\tau +
  \widetilde t_l + t_h + \kappa t_h < t + 1$,
  and the last inequality follows from the hypothesis of the lemma.
  The second expression defining $P$ similarly gives
  \[
    K_p \tau^\beta + t_h^\beta + \lambda^{-1}2^\alpha L^\beta \leq P(\tau +
    \widetilde t_l + t_h + \kappa t_h)^\beta < K_pt^\beta.
  \]
  Finally, the expression defining $Q$ similarly gives
  \begin{align*}
    K_p\tau^\beta + t_h^\beta + t_l^\beta
    +\lambda^{-1}2^\alpha L^\beta
    & \leq K_p\tau^\beta + t_h^\beta + \widetilde t_l^\beta
    +\lambda^{-1}2^\alpha(\kappa t_h)^\beta + 1\\
    &= t_h^\beta\bigl(K_pz^\beta + 1 + r^\beta
    +\lambda^{-1}2^\alpha\kappa^\beta\bigr) + 1\\
    & \leq Q(\tau + \widetilde t_l + t_h + \kappa t_h)^\beta + 1\\
    & \leq Q(t + 1)^\beta + 1\\
    &= Qt^\beta(1 + t^{-1})^\beta + 1 \\
    & \leq \bigl(Q(1 + F^{-1}) + F^{-\beta}\bigr)t^\beta \\
    &< K_t t^\beta.
  \end{align*}

  For the intermediate bound, define
  $f(t_c) = K_p \tau^\beta + t_l^\beta + \lambda2^\alpha t_c^\beta - K_p(\tau +
  t_l + t_h + t_c)^\beta$.
  For $t_c > 0$,
  \[
    f'(t_c) = \beta \left(\lambda2^\alpha t_c^{\beta - 1} - K_p(\tau + t_l + t_h
    + t_c)^{\beta - 1} \right) \geq 0
  \]
  since $\lambda2^\alpha > 1 > K_p$ and $\beta < 1$. By continuity at zero,
  $f$ is nondecreasing on $[0,L]$. The first inequality proved above gives
  $f(L) < 0$, so $f(t_c) \leq f(L) < 0$ for all $t_c \in [0,L]$, including the
  case $L = 0$.
\end{proof}

The arguments above reduce all the conditions in
\Cref{def:real-admissible} to the three inequalities in
\eqref{eq:real-induction-slack}. These depend on the remaining
choices of $F$, $K_p$, $K_t$ as well as the six parameters.
The next lemma verifies these inequalities for the explicit values in
\Cref{prop:explicit-parameters}.

  \begin{lemma}
  \label{lem:parameter-verification}
  The tuple in \eqref{eq:explicit-parameters}, with the auxiliary
  choices in \eqref{eq:explicit-auxiliary-parameters}, satisfies
  \[
    \begin{gathered}
      P(1 + F^{-1}) + F^{-\beta} < K_p,\qquad Q(1 + F^{-1}) + F^{-\beta} <
      K_t,\\
      D + F^{-\beta} < 2^\alpha.
    \end{gathered}
  \]
\end{lemma}

\begin{proof}
  We first identify the points at which the suprema defining $P$ and $Q$
  are attained. Recall that the trigger domain is the union of two segments:
  \[
    \mathcal T
    = \{(z,\rho):0 \leq z \leq \delta\rho\}
    \cup\{(\delta r,r):\rho \leq r\leq1\}.
  \]
  Each expression in \eqref{eq:real-p-coefficient} and
  \eqref{eq:real-q-coefficient} has the form
  \[
    \frac{K_pz^\beta + ar^\beta + b}{(z + 1 + r + \kappa)^\beta},
  \]
  where the coefficients $a$, $b$ are chosen as follows:
  \[
    \begin{array}{c|cc}
      & a & b\\
      \hline
      \text{First expression for }P
      &1&\lambda2^\alpha\kappa^\beta\\
      \text{Second expression for }P
      &0&1 + \lambda^{-1}2^\alpha\kappa^\beta\\
      \text{Expression for }Q
      &1&1 + \lambda^{-1}2^\alpha\kappa^\beta
    \end{array}
  \]
  Consider one row of the table. On the first segment,
  substituting $r = \rho$ yields, for $z \in [0, \delta \rho]$,
  \[
    f(z) = \frac{K_pz^\beta + a\rho^\beta + b}
    {(z + 1 + \rho + \kappa)^\beta}.
  \]
  On the second segment, substituting $z = \delta r$ yields, for $r \in [\rho,
  1]$,
  \[
    g(r) = \frac{(K_p\delta^\beta + a)r^\beta + b}
    {((1 + \delta)r + 1 + \kappa)^\beta}.
  \]
  Both functions are continuous on their closed intervals and differentiable in
  their interiors, since their denominators are positive. Each therefore
  attains its maximum at an endpoint or an interior critical point.
  Differentiation gives
  \[
    f'(z) = \frac{\beta\left[
    K_p(1 + \rho + \kappa)z^{\beta - 1} - (a\rho^\beta + b)
    \right]}{(z + 1 + \rho + \kappa)^{\beta + 1}}
  \]
  and
  \[
    g'(r) = \frac{\beta\left[
    (K_p\delta^\beta + a)(1 + \kappa)r^{\beta - 1} - (1 + \delta)b
    \right]}{((1 + \delta)r + 1 + \kappa)^{\beta + 1}}.
  \]
  Because $0 < \beta < 1$, each numerator is strictly decreasing and has a
  unique positive zero. If the critical point lies in the interval,
  the maximum occurs at the critical point; otherwise, it occurs at the nearest endpoint.
  The maximizing points are therefore
  \[
    z_* = \min\left\{\delta\rho,\,
    \left(\frac{K_p(1 + \rho + \kappa)}{a\rho^\beta + b}
    \right)^{1/(1 - \beta)}\right\}
  \]
  and
  \[
    r_* = \min\left\{1,\max\left\{\rho,\,
    \left(\frac{(K_p\delta^\beta + a)(1 + \kappa)}
    {(1 + \delta)b}\right)^{1/(1 - \beta)}\right\}\right\}.
  \]
  For each row of the table, the supremum over $\mathcal T$ must be
  $\max\{f(z_*),g(r_*)\}$. This means $P$ is the maximum of the four values
  corresponding to the first two rows, and $Q$ is the maximum of the two
  values corresponding to the third row.

  It remains to certify
  \[
    P < \frac{K_p - F^{-\beta}}{1 + F^{-1}},\qquad
    Q < \frac{K_t - F^{-\beta}}{1 + F^{-1}},\qquad
    D + F^{-\beta} < 2^\alpha.
  \]
  The Python program in Listing~\ref{lst:parameter-verifier} verifies these
  bounds using \texttt{mpmath}~1.3.0 interval arithmetic at 80 decimal digits
  of precision. The code uses outward upper bounds for $M$, $D$, $P$, $Q$;
  the strict inequalities verified for these bounds also hold for the
  exact coefficients defined above.
\end{proof}

\begin{proof}[Proof of \Cref{prop:explicit-parameters}]
  The parameter ranges and the identity
  $\varepsilon = 1 - \alpha - \beta = 0.0342855641$ follow directly from
  the exact decimals in \eqref{eq:explicit-parameters} and
  \eqref{eq:explicit-auxiliary-parameters}. The explicit construction
  above verifies the remaining
  auxiliary ranges, the two bounds for unfinished epochs, and the
  three normalized endpoint bounds. \Cref{lem:real-waiting-counts}
  verifies the two phase~1 inequalities.
  \Cref{lem:parameter-verification} establishes all three inequalities
  in \eqref{eq:real-induction-slack}; the first two then allow
  \Cref{lem:real-terminal-bounds} to verify the four phase~2 inequalities.
  Thus every condition in \Cref{def:real-admissible} holds, proving
  admissibility of the six-parameter tuple. Finally,
  $C = 10^{20}/K_p = F^2/K_p$ is a sufficient choice for the induction.
\end{proof}

It remains to verify the resulting calibration exponent, which can be
checked using rational arithmetic alone.

\begin{lemma}
  \label{lem:parameter-calibration-exponent}
  For the parameters in \eqref{eq:explicit-parameters} and $\varepsilon = 1 -
  \alpha - \beta$,
  \[
    \frac23 - \frac{\varepsilon}{9 + 6\varepsilon}
    = \frac{20342855641}{30685711282} < 0.662942288.
  \]
\end{lemma}
\begin{proof}
  Substituting the exact decimals into
  $(2 + \varepsilon)/(3 + 2\varepsilon)$ gives the displayed
  fraction. The strict inequality follows from
  \[
    0.662942288 - \frac{20342855641}{30685711282}
    = \frac{162390413}{958928477562500000} > 0.
  \]
\end{proof}

\subsection{Python code}
\label{app:parameter-code}

Listing~\ref{lst:parameter-verifier} checks that the stated parameters satisfy
all inequalities in
\Cref{def:real-admissible} and computes the resulting calibration exponent.

\begin{lstlisting}[
language=Python,
basicstyle=\ttfamily\footnotesize,
numbers=left,
numberstyle=\tiny,
numbersep=8pt,
xleftmargin=2em,
breaklines=true,
columns=fullflexible,
keepspaces=true,
showstringspaces=false,
caption={Independent verification of the explicit parameters.},
label={lst:parameter-verifier}
]
# Run with Python and mpmath 1.3.0: python3 verify_parameters.py.

from fractions import Fraction
from mpmath import iv

iv.dps = 80
parameters = (
    '0.0082540892', '0.9574603467', '1.9444858156',
    '0.0164237243', '0.4148723339',
    '0.8812155678', '0.9208553424', '9.4880928733',
)
alpha, beta, lam, rho, kappa, K_p, K_t, delta = map(iv.mpf, parameters)
F = iv.mpf('1e10')
p = 1/(1-beta)
s = 1+kappa
v = 2**alpha
rounding = F**(-beta)
inflation = 1+1/F

def clip(x, lower, upper):
    """Enclose min(upper, max(lower, x)) using monotonicity."""
    return iv.mpf([
        min(upper.a, max(lower.a, x.a)),
        min(upper.b, max(lower.b, x.b)),
    ])


def norm(x, y):
    """Coefficient in the weighted concavity inequality."""
    return (x**p+y**p)**(1/p)


def endpoint_upper(a, b):
    """Upper bound from the two maximizing points in Lemma A.1."""
    z = clip((K_p*(rho+s)/(a*rho**beta+b))**p,
             iv.mpf(0), delta*rho)
    r = clip(((K_p*delta**beta+a)*s/((1+delta)*b))**p,
             rho, iv.mpf(1))
    f = (K_p*z**beta+a*rho**beta+b)/(z+rho+s)**beta
    g = ((K_p*delta**beta+a)*r**beta+b)/((1+delta)*r+s)**beta
    return max(f.b, g.b)


# Phase 1: the two cases in the waiting-counts lemma.
x_star = clip(K_p**(-p), iv.mpf(0), 1/delta)
M_tau = norm((K_p+x_star**beta)/(1+x_star)**beta, 1)
M_rho = norm(K_p, (1+rho**beta)/(1+rho)**beta)

# Upper endpoints are exact witnesses, so non-strict checks are rigorous.
M = max(M_tau.b, M_rho.b)
D_one_half = norm(K_t, 1)
D_continuation = norm(M, v/lam)
D = max(D_one_half.b, D_continuation.b)

# The three normalized endpoint coefficients.
P_plus = endpoint_upper(0, 1+v*kappa**beta/lam)
P_minus = endpoint_upper(1, lam*v*kappa**beta)
Q_total = endpoint_upper(1, 1+v*kappa**beta/lam)
P = max(P_plus, P_minus)
Q = Q_total

# Bounds after integer rounding, from the terminal-bounds lemma.
P_rounded = P*inflation+rounding
Q_rounded = Q*inflation+rounding
D_rounded = D+rounding
terminal_minus = P_minus*inflation+rounding
terminal_plus = P_plus*inflation+rounding
terminal_total = Q_total*inflation+rounding

epsilon = 1-alpha-beta
epsilon_exact = 1-Fraction(parameters[0])-Fraction(parameters[1])
exponent = (2+epsilon)/(3+2*epsilon)


# Parameter ranges, in the order of the admissibility definition.
assert alpha > 0
assert 0 < beta < 1
assert 1 < lam <= 2
assert 0 < rho <= 1
assert 0 < kappa <= 1
assert delta > 0

# 1. Auxiliary ranges.
assert F >= 1
assert 0 < K_p < 1
assert K_t > K_p
assert K_p <= M
assert M <= D
assert P > 0
assert Q > 0

# 2. Induction slack.
assert P_rounded < K_p
assert Q_rounded < K_t
assert D_rounded < v

# 3. Unfinished epochs: weighted concavity covers every x,y >= 0.
assert D_one_half.b <= D
assert D_continuation.b <= D

# 4. Phase 1: these two cases certify the waiting inequality.
# The same bounds certify the first-trigger inequality with its +1,
# since (x+1)**beta-x**beta <= 1 for every x >= 0 and 0 < beta < 1.
assert M_tau.b <= M
assert M_rho.b <= M

# 5. Normalized endpoints: the three inequalities over all of T.
assert P_plus <= P
assert P_minus <= P
assert Q_total <= Q

# 6. Phase 2: the three terminal inequalities, in their stated order.
assert terminal_minus < K_p
assert terminal_plus < K_p
assert terminal_total < K_t
# The intermediate inequality follows from the terminal-minus bound
# and monotonicity of its deficit on 0 <= t_c <= L.
assert lam*v >= K_p

# Reduction to calibration and the exact value of epsilon.
assert alpha < beta
assert epsilon_exact == Fraction('0.0342855641')
assert epsilon > 0
assert exponent < iv.mpf('0.662942288')
print('PASS: all admissibility conditions and the calibration exponent.')
\end{lstlisting}

\end{document}